\documentclass[12pt,a4paper]{article}
\usepackage{graphicx}
\let\vec\relax 
\DeclareMathAccent{\vec}{\mathord}{letters}{"7E} 

\usepackage{lmodern}

\usepackage{amssymb}
\usepackage{amsmath,empheq, amsfonts}
\usepackage[section]{placeins} 
\usepackage{mathtools}  
\usepackage{amsthm}
\usepackage{amssymb}
\usepackage{amsthm}
\usepackage{booktabs}
\usepackage{adjustbox}
\usepackage{url}
\usepackage{algorithm} 
\usepackage{multirow} 

\newtheorem{mythm}{Theorem}
\newtheorem{mydef}{Definition}
 
\usepackage{xcolor}
\providecommand{\blue}[1]{\textcolor{black}{#1}}

\usepackage{silence} 
\DeclarePairedDelimiter\floor{\lfloor}{\rfloor}

\DeclareMathOperator{\argmin}{\mbox{argmin}}
\newcommand{\rank}{\mathrm{rank}}
\newcommand{\diag}{\mathrm{diag}}
\newcommand{\trace}{\mathrm{trace}}
\newcommand{\med}{\mbox{median}}
\newcommand{\mad}{\mbox{mad}}
\newcommand{\MD}{\mbox{MD}}
\newcommand{\LD}{\mbox{LD}}

\newcommand{\eps}{\varepsilon}

\newcommand{\hmu}{\hat{\mu}}
\newcommand{\hsigma}{\hat{\sigma}}
\newcommand{\hSigma}{\hat{\Sigma}}

\newcommand{\mF}{\mathcal{F}}

\newcommand{\tF}{\tilde{\mathcal{F}}}
\newcommand{\tK}{\tilde{K}}

\begin{document}
	
\title{Outlier detection in non-elliptical 
       data\\ by kernel MRCD}
	
	
\author{Joachim Schreurs, Iwein Vranckx, 
  Mia Hubert,\\  
	Johan A.K. Suykens, 
	Peter J. Rousseeuw\linebreak \\ \\
  KU Leuven, Belgium}
	
	%

\date{March 29, 2021} 
	
\maketitle
	
\begin{abstract}
The minimum regularized covariance determinant 
method (MRCD) is a robust estimator for multivariate 
location and scatter, which detects outliers by 
fitting a robust covariance matrix to the data. 
Its regularization ensures that the covariance 
matrix is well-conditioned in any dimension.
The MRCD assumes that the non-outlying 
observations are roughly elliptically distributed, 
but many datasets are not of that form. 
Moreover, the computation time of MRCD increases
substantially when the number of variables goes up, 
and nowadays datasets with many variables are common.
The proposed Kernel Minimum Regularized Covariance 
Determinant \mbox{(KMRCD)} estimator addresses 
both issues.
It is not restricted to elliptical data because it
implicitly computes the MRCD estimates in a kernel 
induced feature space.
A fast algorithm is constructed that starts from
kernel-based initial estimates and exploits the 
kernel trick to speed up the subsequent
computations. Based on the KMRCD 
estimates, a rule is proposed to flag outliers. 
The KMRCD algorithm performs well in simulations,
and is illustrated on real-life data.
\end{abstract}

Regularization.
%

\section{Introduction}
\label{sec:introduction}
	
The minimum covariance determinant (MCD) estimator
introduced in 
\cite{rousseeuw1984least,rousseeuw1985multivariate} 
is a robust 
estimator of multivariate location and covariance.				
It forms the basis of robust versions of 
multivariate techniques such as discriminant 
analysis, 
principal component analysis, factor analysis 
and multivariate regression, see e.g. \cite{Hubert:ReviewHighBreakdown,Hubert:WIRE-MCD2} 
for an overview.
The basic MCD method is quite intuitive.
Given a data matrix of $n$ rows with $p$ columns, 
the objective is to find $h < n$ observations 
whose sample covariance matrix has the lowest 
determinant. 
The MCD estimate of location is then the 
average of those $h$ points, whereas the scatter 
estimate is a multiple of their covariance matrix. 
The MCD has good robustness properties. 
It has a high breakdown value, that is, it can 
withstand a substantial number of outliers.
The effect of a small number of potentially far
outliers is measured by its influence function,
which is bounded \cite{croux1999influence}. 

Computing the MCD was difficult at first
but became faster with the algorithm of
\cite{rousseeuw1999fast} and the deterministic 
algorithm DetMCD \cite{hubert2012deterministic}.
An algorithm for $n$ in the millions was recently
constructed \cite{DeKetelaere:RT-DetMCD}.
But all algorithms for the original MCD require
that the dimension $p$ be lower than $h$ in 
order to obtain an invertible covariance matrix.
In fact it is recommended that $n>5p$ in practice 
\cite{rousseeuw1999fast}. 
This restriction implies that the original MCD 
cannot be applied to datasets with more variables
than cases, that are commonly found in spectroscopy 
and areas where sample acquisition is difficult or 
costly, e.g. in the field of omics data. 

A solution to this problem was recently proposed 
in \cite{boudt2020minimum}, which introduced the 
minimum regularized covariance determinant (MRCD) 
estimator. 
The scatter matrix of a subset of $h$ observations
is now a convex combination of its sample 
covariance matrix and a target matrix. 
This makes it possible to use the MRCD estimator 
when the dimension exceeds the subset size. 
But the computational complexity of MRCD still
contains a term $O(p^3)$ from the
covariance matrix inversion, which limits its use 
for high-dimensional data. Another restriction
is the assumption that the non-outlying 
observations roughly follow an elliptical 
distribution.

To address both issues we propose a generalization 
of the MRCD which is defined in a kernel induced 
feature space $\mathcal{F}$, where the proposed 
estimator exploits the kernel trick: the 
$p \times p$ covariance matrix is not calculated 
explicitly but replaced by the calculation of a 
$n \times n$ centered kernel matrix, resulting 
in a computational speed-up in case $n \ll p$. 
Similar ideas can be found in the literature, see 
e.g. \cite{dolia2006minimum,dolia2007kernel} which
kernelized the minimum volume ellipsoid
\cite{rousseeuw1984least,rousseeuw1985multivariate}. 
The results of the KMRCD algorithm with the 
linear kernel $k(x,y) = x^\top y$ and 
radial basis function (RBF) kernel 
$k(x,y) = e^{-\|x-y\|^2/(2\sigma^2)}$ are shown 
in Figure \ref{fig:toy1}. 
This example will be described in detail in
Section \ref{sec:Rapplication}.

The paper is organized as follows.
Section \ref{sec:themcd} describes the MCD and 
MRCD estimators.	
Section \ref{subsec:kmcdintro} proposes the
kernel MRCD method. Section \ref{sec:algo} 
describes the kernel-based initial estimators 
used as well as a kernelized refinement 
procedure, and proves that the optimization 
in feature space is equivalent to an
optimization in terms of kernel matrices. 
The simulation study in Section 
\ref{sec:simulation} confirms the robustness 
of the method as well as the improved 
computation speed when using a linear kernel. 
Section \ref{sec:Rapplication} illustrates
KMRCD on three datasets, and 
Section \ref{sec:conclusion} concludes.
	
\begin{figure}[ht]
\centering
\includegraphics[width=0.46\textwidth]
	{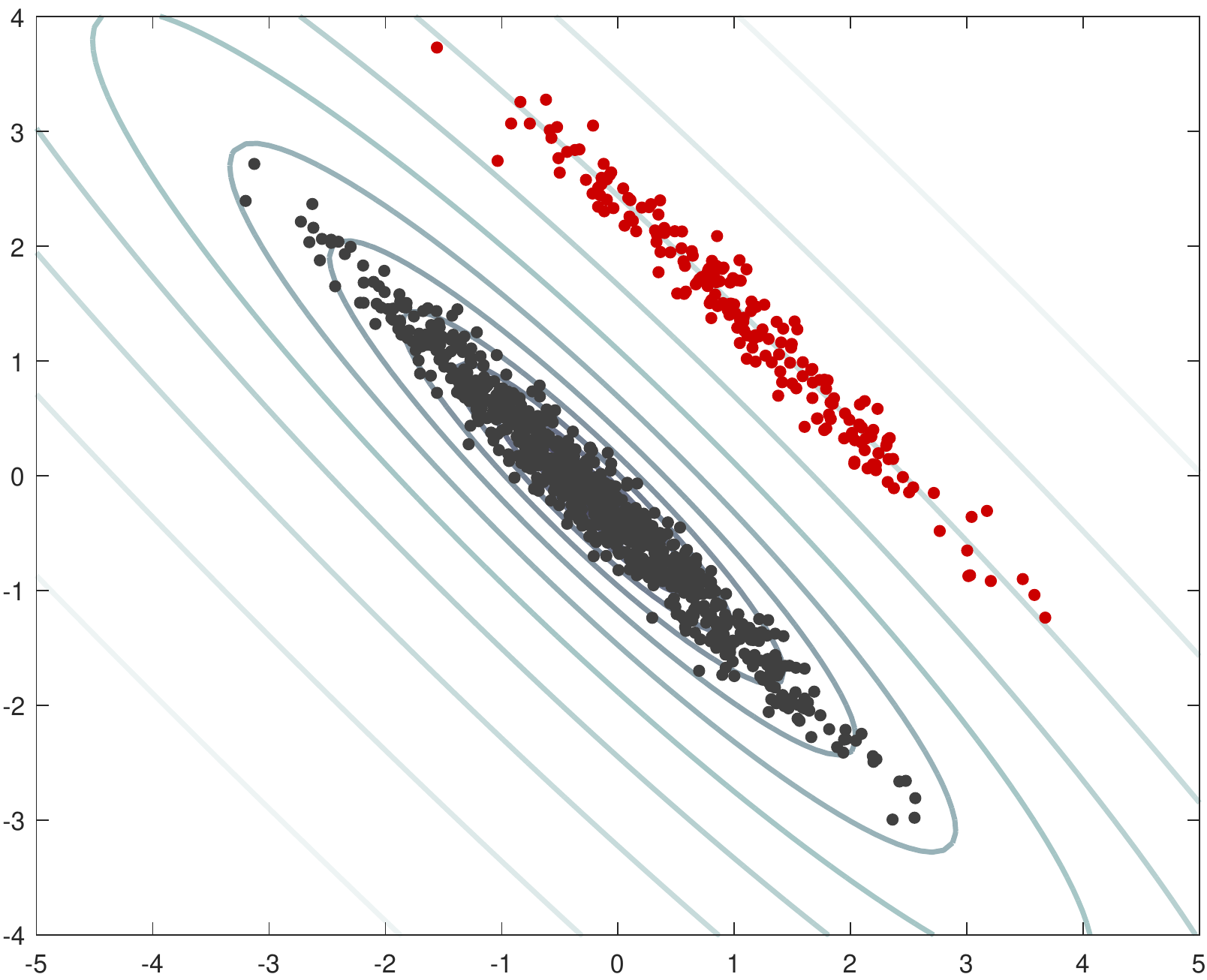} 
\hspace{0.5cm}
\includegraphics[width=0.46\textwidth]
  {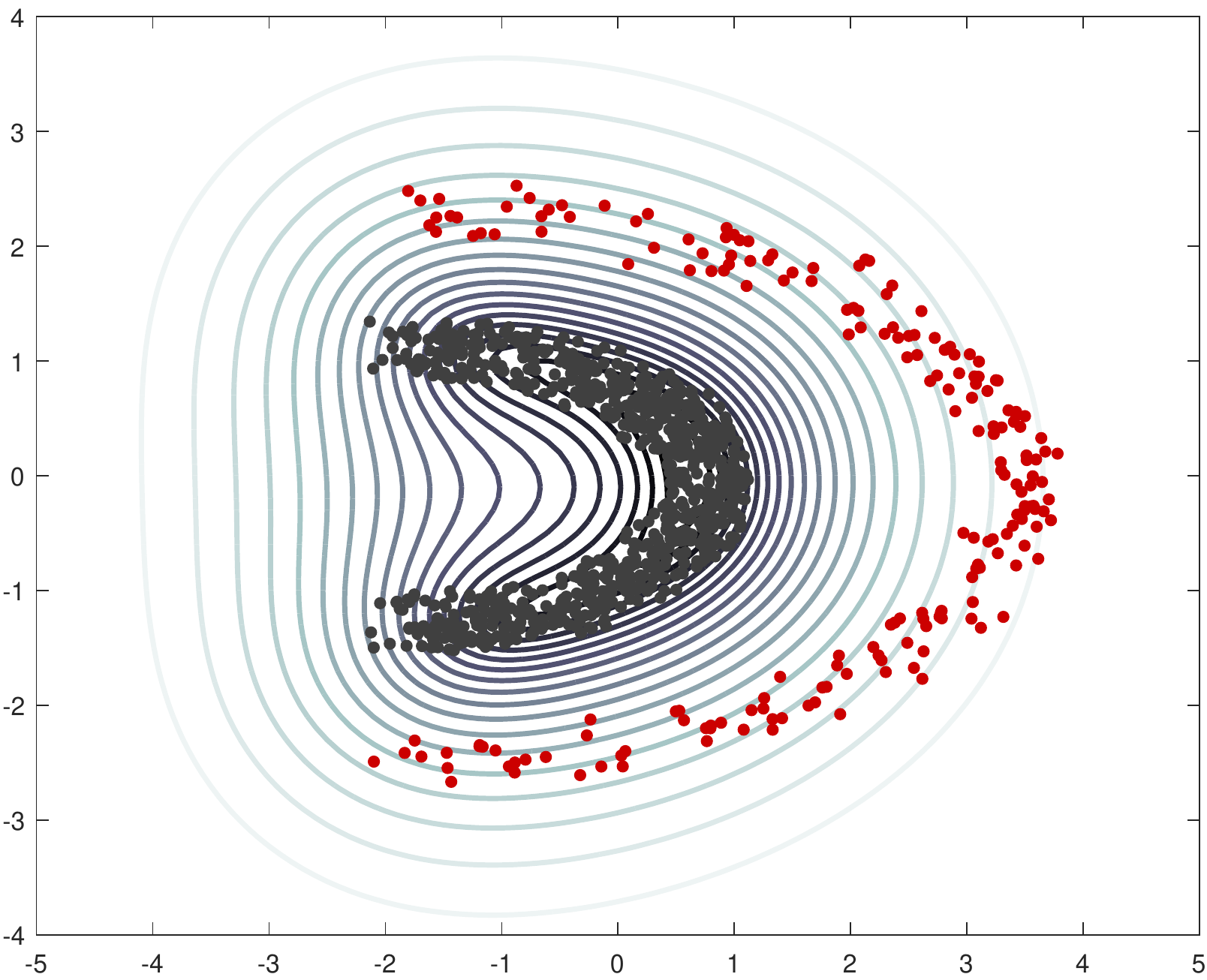}
\caption{Illustration of kernel MRCD on two 
datasets of which the non-outlying part is
elliptical (left) and non-elliptical (right).
Both datasets contain $20\%$ of outlying 
observations. 
The generated regular observations are shown 
in black and the outliers in red. 
In the panel on the left a linear kernel was
used, and in the panel on the right a
nonlinear kernel.
The curves on the left are contours of 
the robust Mahalanobis distance in the 
original bivariate space.
The contours on the right are based on the
robust distance in the kernel-induced
feature space.}
\label{fig:toy1}
\end{figure}

\section{The MCD and MRCD methods}
\label{sec:themcd}	
	
\subsection{The Minimum Covariance Determinant 
estimator}
\label{subsec:mcd}	

Assume that we have a $p$-variate dataset $X$ 
containing $n$ data points, where the $i$-th 
observation 
$x_i = (x_{i1}, x_{i2}, \dots, x_{ip})^\top$ 
is a $p$-dimensional column vector.
We do not know in advance which of these points
are outliers, and they can be located anywhere.
The objective of the MCD method is to find a
set $H$ containing the indices of $|H|=h$ points
whose sample covariance matrix has the lowest 
possible determinant. 
The user may specify any value of $h$ with
$n/2 \leqslant h < n$. The remaining $n-h$ 
observations could potentially be outliers. 
For each $h$-subset $H$ the location estimate 
$c^H$ is the average of these $h$ points:
\begin{equation*}
	c^H = \frac{1}{h} \sum_{i \in H}x_i
\label{eq:plaincentering}
\end{equation*}
whereas the scatter estimate is a multiple of 
their covariance matrix, namely
\begin{equation*}
	c_{\alpha}\hSigma^H = 
	\frac{c_{\alpha}}{h-1} \sum_{i \in H}
	(x_i - c^H) (x_i - c^H)^\top
\label{eq:plaincovar}
\end{equation*}
where $c_{\alpha}$ is a consistency factor 
\cite{croux1999influence} that depends on 
the ratio $\alpha = h/n$.
The MCD aims to minimize the determinant of 
$\hat{\Sigma}^H$ among all $ H \in \mathcal{H}$, 
where the latter denotes the collection of all 
possible sets $H$ with $|H| = h$:
\begin{equation}
\label{eq:mcdobj}
	\hat{\Sigma}_{\mathrm{MCD}} = 
	\underset{H \in \mathcal{H}}\argmin 
	\; \det(\hSigma^H)\;.
\end{equation}		
\blue{Computing the exact MCD has 
combinatorial complexity, so it is 
infeasible for all but tiny datasets.
However, the approximate algorithm FastMCD
constructed in \cite{rousseeuw1999fast}
is feasible.}
FastMCD uses so-called concentration 
steps (C-steps) to minimize \eqref{eq:mcdobj}. 
Starting from any given $\hat{\Sigma}^H$, the 
C-step constructs a more concentrated 
approximation by calculating the Mahalanobis 
distance of every observation based on the 
location and scatter of the current subset $H$:		
\begin{equation*}
\label{eq:smd}
	\MD(x,c^H,\hSigma^H) =
	\sqrt{(x - c^H)^\top 
	(\hat{\Sigma}^H)^{-1} (x - c^H)}\;.
\end{equation*}		
These distances are sorted and the $h$ 
observations with the lowest 
$\MD(x_i,c^H,\hat{\Sigma}^H)$ form 
the new $h$-subset, which is guaranteed to 
have an equal or lower 
determinant \cite{rousseeuw1999fast}. 
The C-step can be iterated until convergence.	

\subsection{The Minimum Regularized Covariance 
            Determinant estimator}
	
The minimum regularized covariance determinant 
estimator (MRCD) is a generalization of the 
MCD estimator to high dimensional 
data \cite{boudt2020minimum}. 
The MRCD subset $H$ is defined by minimizing 
the determinant of the regularized covariance 
matrix $\hat{\Sigma}^H_{\mathrm{reg}}$:
\begin{equation*}
	\hSigma_{\mathrm{MRCD}} =
	\underset{H \in \mathcal{H}}\argmin
	\left(\det(\hSigma^H_{\mathrm{reg}})\right),
\end{equation*}
where the regularized covariance matrix is 
given by
\begin{equation*}
	\hSigma^H_{\mathrm{reg}} =\rho T + 
	(1-\rho) c_{\alpha} \hat{\Sigma}^H
\end{equation*}
with $0 < \rho < 1$ and $T$ a predetermined 
and well-conditioned symmetric and positive 
definite target matrix. 
The determination of $\rho$ is done in a 
data-driven way such that 
$\hat{\Sigma}_{\mathrm{MRCD}}$ has a 
condition number at most $\kappa$, for 
which \cite{boudt2020minimum} proposes 
$\kappa = 50$. The MRCD algorithm 
starts from six robust, well-conditioned 
initial estimates of location and scatter, 
taken from the DetMCD 
algorithm \cite{hubert2012deterministic}. 
Each initial estimate is followed by 
concentration steps, and at the end the 
subset $H$ with the lowest determinant is
kept.
\blue{Note that approximate algorithms like
FastMCD and MRCD are much faster than 
exhaustive enumeration, but one can no 
longer formally prove a breakdown value.
Fortunately, simulations confirm the high 
robustness of these methods. Also note 
that such approximate algorithms are 
guaranteed to converge, because they iterate 
C-steps starting from a finite number of 
initial fits. The algorithm
may converge to a local minimum of the 
objective rather than its global minimum, but 
simulations have confirmed the accuracy of
the result.}
	
\section{The Kernel MRCD Estimator} 
\label{subsec:kmcdintro}

We now turn our attention to kernel 
transformations \cite{scholkopf2002learning}, 
formally defined as follows.
	
\begin{mydef}
A function $k:\mathcal{X} \times \mathcal{X}
\rightarrow \mathbb{R}$ is called a kernel 
on $\mathcal{X}$ iff there exists a real
Hilbert space $\mF$ and a map 
$\phi: \mathcal{X} \rightarrow \mF$ 
such that for all $x$, $y$ in $\mathcal{X}$:
\begin{equation*}
	k(x, y)=\langle\phi(x),\phi(y)\rangle,
\end{equation*}
where $\phi$ is called a feature map and 
$\mathcal{F}$ is called a feature space.
\end{mydef} 
We restrict ourselves to positive 
semidefinite (PSD) kernels.
A symmetric function 
$k:\mathcal{X} \times \mathcal{X} \to 
 \mathbb{R}$ 
is called PSD iff 
$\sum_{i=1}^n \sum_{j=1}^n c_i c_j 
 k(x_i, x_j) \geqslant 0$
for any $x_1, \dots, x_n$ in $\mathcal X$ and
any $c_1, \dots, c_n$ in $\mathbb{R}$. 
Given an $n \times p$ dataset $X$, 
its kernel matrix is defined as 
$K = \Phi \Phi^\top$ with 
$\Phi = [\phi(x_1),...,\phi(x_n)]^\top$. 
The use of kernels makes it possible to operate 
in a high-dimensional, implicit feature space 
without computing the coordinates of the data 
in that space, but rather by replacing 
inner products by kernel matrix entries. 
A well known example is given by kernel 
PCA \cite{scholkopf1998nonlinear}, where linear 
PCA is performed in a kernel-induced feature 
space $\mathcal{F}$ instead of the original 
space $\mathcal{X}$. 
Working with kernel functions has the 
advantage that non-linear kernels enable the 
construction of non-linear models. 
Note that the size of the kernel matrix is 
$n \times n$, whereas the covariance 
matrix is  $p \times p$. The latter is an
advantage when dealing with datasets
for which $n \ll p$, for then the memory and 
computational requirements are considerably 
lower.
	
Given an $n \times p$ dataset 
$X = \{x_1, \ldots, x_n\} $ we thus
get its image $\{\phi(x_1),\ldots\phi(x_n)\}$
in feature space, where it has the average
\begin{equation*}
	c_\mathcal{F} = \frac{1}{n} 
	   \sum_{i=1}^{n} \phi(x_i)\;.
\end{equation*}
Note that the dimension of the feature space 
$\mathcal{F}$ may be infinite.
However, we will restrict ourselves to the 
subspace $\tilde{\mathcal{F}}$ spanned by 
$\{\phi(x_1) - c_\mathcal{F},\dots,
 \phi(x_n) - c_\mathcal{F}\}$ 
so that 
$m:=\mathrm{dim}(\tilde{\mathcal{F}}) 
\leqslant n-1$. 
In this subspace the points
$\phi(x_i) - c_\mathcal{F}$ 
thus have at most $n-1$ coordinates.
The covariance matrix in the feature space 
given by
\begin{equation*}
	\hSigma_\mathcal{F} = \frac{1}{n-1} 
	\sum_{i =1}^n (\phi(x_i) - c_\mathcal{F}) 
	(\phi(x_i) - c_\mathcal{F})^\top
\end{equation*}
is thus a matrix of size at most 
$(n-1) \times (n-1)$.
Note that the covariance matrix is centered 
but the original kernel matrix is not. 
Therefore we construct 
the centered kernel matrix $\tilde{K}$ by
\begin{align} \label{eq:centerKh}
	\tilde{K}_{ij} &= 
	   \Big(\phi(x_i) - 
	   \frac{1}{n}\sum_{\ell = 1}^n
	   \phi(x_\ell)\Big)^\top
		 \Big(\phi(x_j) - 
	   \frac{1}{n}\sum_{\ell'=1}^n
	   \phi(x_{\ell'})\Big) \nonumber \\
	&= K_{ij} - \frac{1}{n}
	   \sum_{\ell = 1}^n K_{\ell j} - 
		 \frac{1}{n}
		 \sum_{\ell' = 1}^n K_{i\ell'} + 
		 \frac{1}{n^2}\sum_{\ell = 1}^n
	   \sum_{\ell' = 1}^n K_{\ell \ell'} 
		 \nonumber \\
	&= \Big(K - 1_{nn} K - K 1_{nn} +
	   1_{nn} K 1_{nn}\Big)_{ij}
\end{align}
where $1_{nn}$ is the $n\times n$ matrix with 
all entries set to $1/n$. 
Note that the centered kernel matrix is equal 
to $\tilde{K} =
 \tilde{\Phi} \tilde{\Phi}^\top$ with 
$\tilde{\Phi} = [\phi(x_1) - c_\mathcal{F},
 \ldots,\phi(x_n) - c_\mathcal{F}]^\top$
and is PSD by construction. 
The following result is due to 
\cite{scholkopf1998nonlinear}.

\begin{mythm} \label{thm:equalrank}
Given an $n \times p$ dataset $X$, the 
sorted eigenvalues of the covariance matrix 
$\hat{\Sigma}_\mF$ and those of the 
centered kernel matrix $\tilde{K}$ satisfy
\begin{equation*}
	\lambda_j^{\hat{\Sigma}_\mF} =
	\frac{\lambda_j^{\tilde{K}}}{n-1} 
\end{equation*}
for all $j=1, \ldots, m$ where
$m=\rank(\hSigma_\mathcal{F})$. 
\end{mythm}
	
\begin{proof}
[Proof of Theorem \ref{thm:equalrank}]
The eigendecomposition of the centered 
kernel matrix $\tilde{K}$ is
\begin{equation*}
	\tilde{K} = 
	  \tilde{\Phi} \tilde{\Phi}^\top
		= V \Lambda V^\top 
\end{equation*}
where $\Lambda = \diag(\lambda_1,
 \ldots,\lambda_n)$ with 
$\lambda_1 \geqslant \ldots \geqslant 
\lambda_n$\,.
The eigenvalue $\lambda_j$ and 
eigenvector $v_j$ satisfy
\begin{equation*}
	\tilde{\Phi} \tilde{\Phi}^\top v_j
	= \lambda_j v_j 
\end{equation*}  
for all $j=1, \ldots, m$.
Multiplying both sides by 
$\tilde{\Phi}^\top/(n-1)$ gives
\begin{equation*}
		\left(\frac{1}{n-1} 
		\tilde{\Phi}^\top 
		\tilde{\Phi}\right) 
		(\tilde{\Phi}^\top v_j) = 
		\frac{\lambda_j}{n-1} 
		(\tilde{\Phi}^\top v_j).
\end{equation*} 	
Combining the above equations results in
\begin{equation*}
		\hat{\Sigma}_\mathcal{F} 
		v^{\hat{\Sigma}_\mathcal{F}}_j = 
		\frac{\lambda_j}{n-1} 
		v^{\hat{\Sigma}_\mathcal{F}}_j
\end{equation*} 
for all $j=1, \ldots, m$
where $v^{\hat{\Sigma}_\mathcal{F}}_j =
 (\tilde{\Phi}^\top v_j)$ is the $j$-th 
eigenvector of $\hat{\Sigma}_\mathcal{F}$. 
The remaining eigenvalues of the covariance 
matrix, if any, are equal to zero. 
\end{proof}
	
The above result can be related to a 
representer theorem for kernel 
PCA \cite{alzate2008kernel}. 
It shows that the nonzero eigenvalues of the 
covariance matrix are proportional to the 
nonzero eigenvalues of the centered kernel 
matrix, thus proving that 
$\hat{\Sigma}_\mathcal{F}$ and 
$\tilde{K}$ have the same rank.

What would a kernelized MCD estimator look like? 
It would have to be equivalent to applying the 
original MCD in the feature space, so that in
case of the linear kernel the original MCD is
obtained. The MCD estimate for location in 
the subspace $\tF$ is
\begin{equation*}\label{eq:centeringEQ}
 c_\mF^H = \frac{1}{h} \sum_{i \in H} \phi(x_i)
\end{equation*}
whereas the covariance matrix now equals
\begin{equation*}	\label{eq:kcov}
	\hSigma_\mF^H = \frac{1}{h-1}\sum_{i \in H}
	(\phi(x_i) - c_\mF^H) 
	(\phi(x_i) - c_\mF^H)^\top.
\end{equation*}
Likewise, the robust distance becomes
\begin{equation*} \label{eq:ksmd}
	\MD(\phi(x),c_\mF^H,\hSigma_\mF^H) =
	(\phi(x) - c_\mF^H)^\top 
	(\hSigma^{H}_\mF)^{-1} 
	(\phi(x) - c_\mathcal{F}^H)\;.
\end{equation*}	
In these formulas the mapping function $\phi$
may not be known, but that is not necessary 
since we can apply the kernel trick.
More importantly, the covariance matrix may
not be invertible as the $\phi(x_i) - c_\mF^H$
lie in a possibly high-dimensional space $\tF$. 
We therefore propose to apply MRCD in $\tF$ in 
order to make the covariance matrix invertible.  
Let $\tilde{\Phi}_H$ be the row-wise 
stacked matrix
\begin{equation*}
		\tilde{\Phi}_H = 
		[\phi(x_{i(1)}) - c_\mF^H,\ldots,
		 \phi(x_{i(h)}) - c_\mF^H]^\top
\end{equation*}	
where $i(1),\ldots,i(h)$ are the indices in $H$. 
For any $0 < \rho < 1$ the regularized 
covariance matrix is defined as
\begin{equation*}
	\label{eq:sigmareg}
	\hSigma^H_{\mathrm{reg}} = 
	(1-\rho)\hSigma^H_\mF + \rho I_m =
	\frac{1-\rho}{h-1} \tilde{\Phi}_H^\top 
	\tilde{\Phi}_H + \rho I_m
	\end{equation*}
where $I_m$ is the identity matrix 
in $\tF$. 
The KMRCD method is then defined as
\begin{equation}\label{eq:kmcdobjective}	
	\hSigma_{\mathrm{KMRCD}} = 
	\underset{H \in \mathcal{H}}
	\argmin \det(\hSigma_{\mathrm{reg}}^H)
\end{equation}	
where $\mathcal{H}$ is the collection of subsets 
$H$ of $\{1,\ldots,n \}$ such that $|H| = h$ 
and $\hat{\Sigma}^H$ is of maximal rank, i.e. 
$\rank(\hat{\Sigma}^H) = 
\mathrm{dim}(\mathrm{span}
  (\phi(x_{i(1)}) - c_\mF^H, \ldots,
	\phi(x_{i(h)}) - c_\mF^H)) = q$ 
with $q:= \min(m,h-1)$. 
We can equivalently say that the $h$-subset 
$H$ is in general position.
The corresponding regularized kernel matrix is
\begin{equation} \label{eq:egKernelMatrix}
	\tK_{\mathrm{reg}}^H = (1-\rho)\tK^H + 
	                       (h-1) \rho I_h
\end{equation}
where $\tilde{K}^H = \tilde{\Phi}_H
  \tilde{\Phi}_H^{\operatorname{T}}$
denotes the centered kernel matrix of $h$ rows,
that is, \eqref{eq:centerKh} with $n$ replaced 
by $h$. 
The MRCD method in feature space $\tF$ 
minimizes the determinant in 
\eqref{eq:kmcdobjective} in $\tF$.
But we would like to carry out an optimization 
on kernel matrices instead.
The following theorem shows that this is
possible.
	
\begin{mythm} \label{thm:equalobj}
Minimizing $\det(\hSigma^H_\mathrm{reg})$ over 
all subsets H in $\mathcal{H}$ is equivalent to 
minimizing $\det(\tK^H_\mathrm{reg})$ over all 
$h$-subsets $H$ with $\rank(\tK^H) = q.$
		
\begin{proof}
[Proof of Theorem \ref{thm:equalobj}]
From Theorem \ref{thm:equalrank} it follows 
that the nonzero eigenvalues of 
$\hat{\Sigma}^H_\mathcal{F}$ and $\tilde{K}^H$ 
are related by 
$\lambda_j^{\hat{\Sigma}_\mathcal{F}^H} =
 \frac{1}{h-1}\lambda_j^{\tilde{K}^H}$. 
If $H$ belongs to $\mathcal{H}$, 
$\hat{\Sigma}_\mathcal{F}^H$ has exactly $q$ 
nonzero eigenvalues so $\tilde{K}^H$ also has 
rank $q$, and vice versa. The remaining 
$m-q$ eigenvalues of $\hat{\Sigma}^H$ are zero, 
as well as the remaining $h-q$ eigenvalues of 
$\tilde{K}^H$. 
Now consider the regularized matrices
\begin{equation*}
	\hSigma^H_\mathrm{reg} = 
	(1-\rho) \hSigma_\mF^H + \rho I_m
\end{equation*}
and
\begin{equation*}
	\tK^H_\mathrm{reg} = 
	(1-\rho) \tK^H + (h-1)\rho I_h\;.
\end{equation*}
Computing the determinant of both matrices as 
a product of their eigenvalues yields:
\begin{equation*}
	\det(\hSigma^H_\mathrm{reg}) = 
	\rho^{m-q} \prod_{j=1}^q ((1-\rho)
	\lambda_j^{\hat{\Sigma}^H_\mathcal{F}}
	+ \rho)
\end{equation*}
and
\begin{align*}
	\det(\tK^H_\mathrm{reg}) &= 
	\rho^{h-q} \prod_{j=1}^q ((1-\rho)
	\lambda_j^{\tK^H} + (h-1)\rho) \\
	&= \rho^{h-q} \prod_{j=1}^q (h-1)((1-\rho)
	  \lambda_j^{\hSigma_\mF^H} + \rho)\\
	&= \frac{\rho^{h-q}}{\rho^{m-q}} (h-1)^q 
	   \det(\hSigma^H_\mathrm{reg})\;.
\end{align*}
Therefore	
$\mathrm{det}(\tilde{K}^H_\mathrm{reg}) = 
\rho^{h-m} (h-1)^q 
\det(\hat{\Sigma}^H_\mathrm{reg})$ in which 
the proportionality factor is constant,
so the optimizations are equivalent.
\end{proof}
\end{mythm}

Following \cite{haasdonk2009classification}
we can also express the robust Mahalanobis 
distance in terms of the regularized kernel 
matrix, by 
\begin{align}\label{eq:regMahal}
  \MD(\phi(x),c_\mF^H,
	    \hSigma_{\mathrm{reg}}^H) 
	&= \sqrt{(\phi(x) - c_\mF^H)^\top
		 (\hSigma_{\mathrm{reg}}^H)^{-1}
		 (\phi(x) - c_\mF^H)} \nonumber \\
	&= \sqrt{\frac{1}{\rho}\left(
	   \tilde{k}(x,x)-(1-\rho)\tilde{k}(H,x)^\top 
		 (\tilde{K}_{\mathrm{reg}}^H)^{-1} 
		 \tilde{k}(H,x)\right)} 	
\end{align}	
where $\tilde{k}(x,x) = (\phi(x) - c_\mF^H)^\top
(\phi(x) - c_\mF^H)$ is a special case of the
formula 
$\tilde{k}(x,y) = k(x,y) 
  - \sum_{i \in H} k(x_i,x)
 	- \sum_{i \in H} k(x_i,y)
  - \sum_{i \in H} \sum_{j \in H} k(x_i,x_j)$
for $x=y$. The notation
$\tilde{k}(H,x)$ stands for the column vector
$\tilde{\Phi}_H (\phi(x) - c_\mF^H) = 
[\tilde{k}(x_{i(1)},x),	\ldots , 
 \tilde{k}(x_{i(h)},x)]^\top$
in which $i(1), \ldots, i(h)$ are the members
of $H$.
This allows us to calculate the Mahalanobis 
distance in feature space from the kernel 
matrix, and consequently to perform the 
C-step procedure on it. 
Note that \eqref{eq:regMahal} requires
inverting the matrix
$\tilde{K}_{\mathrm{reg}}^H$ instead of
the matrix
$\hSigma_{\mathrm{reg}}^H$.
		
The C-step theorem of the MRCD in
\cite{boudt2020minimum} shows that when you
select a new $h$-subset as those $i$ for which
the Mahalanobis distance relative to the old
$h$-subset is smallest, the regularized
covariance determinant of the new $h$-subset
is lower than or equal to that of the old one.
In other words, C-steps lower the objective 
function of MRCD. 
Using Theorem \ref{thm:equalobj}, this 
C-step theorem thus also extends to the kernel 
MRCD estimator.

\section{The Kernel MRCD Algorithm} 
\label{sec:algo}

This section introduces the elements of the 
kernel MRCD algorithm.
If the original data comes in the form of
an $n \times p$ dataset $X$, we start by 
robustly standardizing it.
For this we first compute the univariate 
reweighted MCD estimator of 
\cite{Rousseeuw:RobReg} with coverage 
$h = [n/2] + 1$ to obtain estimates of the 
location and scatter of each variable, which 
are then used to transform $X$ to $z$-scores.
The kernel matrix $K$ is then computed
from these $z$-scores.
\blue{Note, however, that the data can come 
in the form of a kernel matrix that was not
derived from data points with coordinates.
For instance, a so-called string kernel
can compute similarities between texts, such 
as emails, without any variables or
measurements. Such a kernel basically 
compares the occurrence of strings of 
consecutive letters in each text.}
Since the KMRCD method does all its 
computations on the kernel matrix, it can 
also be applied to such data.

\subsection{Initial estimates}
\label{subsec:initialestimates}	

The MRCD estimator needs initial
$h$-subsets to start C-steps from.
In the original FastMCD algorithm of
\cite{rousseeuw1999fast} the initial 
$h$-subsets were obtained by drawing
random $(p + 1)$-subsets out of the $n$
data points. For each its empirical mean 
and covariance matrix were computed as 
well as the resulting Mahalanobis distances
of all points, after which the subset with 
the $h$ smallest distances was obtained.
However, this procedure would not be 
possible in situations where $p > n$
because Mahalanobis distances require
the covariance matrix to be invertible.
The MRCD method instead starts from a
small number of other initial estimators, 
inherited from the DetMCD algorithm in 
\cite{hubert2012deterministic}.

For the initial $h$-subsets in KMRCD we
need methods that can be kernelized.
We propose to use four such initial 
estimators, the combination of which has 
a good chance of being robust against
different contamination types. 
Since initial estimators can be inaccurate, 
a kernelized refinement step will be 
applied to each. 
We will describe these methods in turn.

The first initial method is based on the 
concept of spatial median.
For data with coordinates, the spatial 
median is defined as the point $m$ that 
has the lowest total Euclidean distance
$\sum_i ||x_i-m||$ to the data points.
This notion also makes sense in the kernel
context, since Euclidean distances in the
feature space can be written in terms of
the inner products that make up the kernel
matrix. 
\blue{The spatial median in coordinate 
space is often computed by the Weiszfeld
algorithm and its extensions, see
e.g. \cite{vardi2000L1}.
A kernel algorithm for the spatial median
was provided in
\cite{debruyne2010detecting}.
It writes the spatial median $m_\mathcal{F}$
in feature space as} a convex 
combination of the $\phi(x_i)$:
\begin{equation*}
	m_\mF = \sum_{i=1}^{n} 
	   \gamma_i \phi(x_i) 
	\label{eq:spatialmedian}
\end{equation*}
in which the coefficients 
$\gamma_1, \dots, \gamma_n$ are unknown.
The Euclidean distance of each observation 
to $m_\mF$ is computed as the
square root of
\begin{align} \label{eq:edistance}
	||\phi(x_i) - m_\mF||^2 
	&= || \phi(x_i) - \sum_{j=1}^n 
	    \gamma_j \phi(x_j)||^2 \nonumber \\
	&= || \phi(x_i)||^2 + || \sum_{j=1}^n 
	   \gamma_j \phi(x_j)||^2 - 
		 2 \langle \phi(x_i),
		 \sum_{j=1}^n \gamma_j \phi(x_j)\rangle
		 \nonumber \\ 
	&= k(x_i,x_i) + \sum_{j=1}^n \sum_{\ell=1}^n 
	   \gamma_j \gamma_\ell k(x_j,x_\ell) - 
		 2 \sum_{j=1}^n\gamma_j k(x_i,x_j)
\end{align}
and the coefficients $\gamma_1,\dots,\gamma_n$ 
that minimize $\sum_i ||\phi(x_i)-m_\mF||$
are obtained by an iterative procedure
described in Algorithm 2 in 
Section A.1 
of the Supplementary Material.
The first initial $h$-subset $H$ is then given
by the objects with the $h$ lowest values of
\eqref{eq:edistance}. 
Alternatively, $H$ is described by a 
weight vector $w = (w_1,\ldots,w_n)$ of
length $n$, where 
\begin{equation}
\label{eq:hardReject}
		w_i := \begin{cases}
		1 \quad & \text{if } i\in H \\
		0 \quad &\text{otherwise}.
\end{cases}
\end{equation}
The initial location estimate $c_\mF$ in 
feature space is then the weighted mean 
\begin{equation} \label{eq:wsumloc}
	c_\mathcal{F} = \frac{\sum_{i=1}^n 
	w_i\phi(x_i)}{\sum_{i=1}^n\! w_i} \;.
\end{equation} 
\blue{The initial covariance estimate 
$\hSigma_\mF$ is the weighted covariance 
matrix 
\begin{equation} \label{eq:wsumscatter}
	\hSigma_\mF = \frac{1}{\sum_{i=1}^n\! u_i}
	\, \tilde{\Phi}^\top 
	\mathrm{diag}(u_1,\ldots,u_n)\,\tilde{\Phi}
\end{equation}
given by covariance weights 
$(u_1,\ldots,u_n)$ that in general may 
differ from the location weights 
$(w_1,\ldots,w_n)$. But for the spatial 
median initial estimator one simply 
takes $u_i := w_i$ for all $i$.}
	
The second initial estimator is based on the 
Stahel-Donoho outlyingness (SDO) of
\cite{stahel1981breakdown,donoho1982breakdown}. 
In a space with coordinates it involves
projecting the data points on many unit
length vectors (directions).
We compute the kernelized SDO  
\cite{debruyne2010robust} of all observations
and determine an $h$-subset as the indices
of the $h$ points with lowest outlyingness.
This is then converted to weights $w_i$
as in \eqref{eq:hardReject}, \blue{and we 
put $u_i := w_i$ again}. The entire
procedure is listed as Algorithm 3
in the Supplementary Material.
	
The third initial $h$-subset is based on 
spatial ranks \cite{debruyne2009robustified}. 
The spatial rank of $\phi(x_i)$ with respect 
to the other feature vectors is defined as:
\begin{align} \label{eq::srank}
	R_i 
  &= \frac{1}{n}\left\|\sum_{j \neq i} 
	   \frac{\phi\left(x_{i}\right)-
	   \phi\left(x_{j}\right)}
	   {\left\|\phi\left(x_{i}\right)-
	   \phi\left(x_{j}\right)\right\|}\right\|
	   \nonumber \\
	&= \frac{1}{n}\left[\left(\sum_{j\neq i} 
	   \frac{\phi\left(x_{i}\right)-
	   \phi\left(x_{j}\right)}
	   {\left\|\phi\left(x_{i}\right)-
	   \phi\left(x_{j}\right)\right\|}\right)
	   ^\top
	   \left(\sum_{\ell\neq i} 
	   \frac{\phi\left(x_{i}\right)-
	   \phi\left(x_\ell \right)}
	   {\left\|\phi\left(x_{i}\right)-
	   \phi\left(x_\ell \right)\right\|}\right)
		 \right]^{\frac{1}{2}} \nonumber \\ 
	&= \frac{1}{n}\left[\sum_{j \neq i} 
	   \sum_{\ell \neq i} 
	   \frac{k(x_{i}, x_{i})-k(x_{i}, x_{j})-
	   k(x_{i}, x_\ell) + k(x_{j}, x_\ell)}
		 {\alpha (x_i,x_j) \alpha (x_i,x_\ell)}
		 \right]^{\frac{1}{2}} 
\end{align}
where $\alpha (x_i,x_j) = [k(x_{i}, x_{i})+
k(x_{j}, x_{j})-
2 k(x_{i}, x_{j})]^{\frac{1}{2}}$\;. 
If $R_i$ is large, this indicates that 
$\phi(x_i)$ lies further away from the bulk 
of the data than most other feature 
vectors. In this sense, the 
values $R_i$ represent a different measure 
of the outlyingness of $\phi(x_i)$ in the 
feature space. We then consider the $h$ 
lowest spatial ranks, \blue{yielding the 
location weights $w_i$ by 
\eqref{eq:hardReject}, and put 
$u_i := w_i$\;}. The complete procedure is 
Algorithm 4 
in the Supplementary Material.
\blue{Note that this algorithm is closely
related to the depth computation
in~\cite{chen2009spatial} which 
appeared in the same year 
as~\cite{debruyne2009robustified}.}

The last initial estimator is a 
generalization of the spatial sign 
covariance matrix \cite{visuri2000sign} 
(SSCM) to the feature space $\mF$. 
For data with coordinates, one first 
computes the spatial median $m$ 
described above.
The SSCM then carries out a radial 
transform which moves all data points 
to a sphere around $m$, followed by 
computing the classical product 
moment of the transformed data:
\begin{equation*}
	\hat{\Sigma}^{\mathrm{SSCM}} = 
	\frac{1}{n-1} \sum_{i=1}^{n} 
	\frac{\left(x_i-m\right)}{||x_i-m||}
	\frac{\left(x_i-m\right)^\top}
  	{||x_i-m||}\;.
	\end{equation*}
The kernel spatial sign covariance 
matrix \cite{debruyne2010detecting}
is defined in the same way, by
replacing $x_i$ by $\phi(x_i)$ and
$m$ by $m_\mF = \sum_{i=1}^{n} 
 \gamma_i \phi(x_i)$.
We now have two sets of weights. 
\blue{For location we use the weights 
$w_i = \gamma_i$ of the spatial
median and apply \eqref{eq:wsumloc}.
But for the covariance matrix we 
compute the weights 
$u_i = 1/||\phi(x_i) - m_\mF||$ with
the denominator given by 
\eqref{eq:edistance}.
Next, we apply \eqref{eq:wsumscatter}
with these $u_i$\;.}
The entire kernel SSCM procedure is 
listed as Algorithm 5 
in the Supplementary Material.
Note that kernel SSCM uses continuous
weights instead of zero-one weights. 

\subsection{The refinement step}
\label{subsec:refinement}	

It happens that the eigenvalues of initial
covariance estimators are inaccurate.
In \cite{maronna2002robust} this was
addressed by re-estimating the eigenvalues,
and \cite{hubert2012deterministic}
carried out this refinement step for all
initial estimates used in that paper.
In order to employ a refinement step in
KMRCD we need to be able to kernelize it.
We will derive the equations for the 
general case of a location estimator 
given by a weighted sum 
\eqref{eq:wsumloc} and a scatter matrix 
estimate given by a weighted covariance 
matrix \eqref{eq:wsumscatter} so it can
be applied to all four initial estimates.
We proceed in four steps.
\begin{enumerate}
\item The first step consists of 
  projecting the uncentered data on the 
	eigenvectors $V_\mF$ of the initial 
	scatter estimate $\hSigma_\mF$:
\begin{equation}
\label{eq:refin_B}
	B = \Phi V_\mF = 
	\Phi \tilde{\Phi}^\top D^{\frac{1}{2}}V
	= (K - K w 1_n^\top)D^{\frac{1}{2}}V,
\end{equation}
  where $D = \mathrm{diag}(u_1,\ldots,u_n)/ 
	(\sum_{i=1}^n\! u_i)$, 
	$1_{n} = [1,\ldots,1]^\top$, and 
  $V_\mF = \tilde{\Phi}^\top 
  D^{\frac{1}{2}}V$ with $V$ the 
  normalized eigenvectors of the weighted 
  centered kernel matrix $\hat{K} = 
  (D^{\frac{1}{2}}\tilde{\Phi})
  (D^{\frac{1}{2}}\tilde{\Phi})^\top =  
  D^{\frac{1}{2}}\tilde{K}D^{\frac{1}{2}}$.
\vspace{0.5em}
\item Next, the covariance matrix is 
  re-estimated by
\begin{equation*}
\label{eq:refin_Cov}
	\Sigma^{*}_\mF = 
	V_\mF L V_\mF^\top = 
	\tilde{\Phi}^\top D^{\frac{1}{2}}V L
	V^\top D^{\frac{1}{2}}\tilde{\Phi}\;,
\end{equation*}
  where 
  $L=\diag(Q_{n}^{2}\left(B_{.1}\right), 
  \ldots, Q_{n}^{2}\left(B_{.n}\right))$ in
  which $Q_{n}$ is the scale estimator of 
  Rousseeuw and Croux 
  \cite{rousseeuw1993alternatives} and
	$B_{.j}$ is the $j$-th column of $B$.
\vspace{0.5em}
\item The center is also re-estimated, by
\begin{equation*}
		c^*_\mF = (\Sigma_\mF^*)^{\frac{1}{2}} 
		\med(\Phi (\Sigma_\mF^*)^{-\frac{1}{2}})
\end{equation*}
  where $\med$ stands for the spatial median. 
  This corresponds to using a modified 
  feature map $\phi^{*}(x) = 
  \phi(x)(\Sigma_\mF^*)^{-\frac{1}{2}}$ for 
  the spatial median or running 
	Algorithm 2 
	with the modified kernel matrix 
\begin{align} \label{eq:refin_Kadj}
	K^* &= 	\Phi \tilde{\Phi}^\top 
	D^{\frac{1}{2}}V L^{-1} V^\top 
	D^{\frac{1}{2}}\tilde{\Phi} \Phi^\top 
	\nonumber\\ 
	&= (K - K w 1_n^\top) D^{\frac{1}{2}}
	    V L^{-1}D^{\frac{1}{2}}
			(K -  K w 1_n^\top)^\top.		
\end{align}
  Transforming the spatial median gives us 
	the desired center:
\begin{equation*}	\label{eq:refin_loc}
	c^*_\mF = (\Sigma^*_\mF)^{\frac{1}{2}} 
	\sum_{i=1}^n (\Sigma^*_\mF)^{-\frac{1}{2}} 
	\gamma^*_i \phi(x_i)
	= \sum_{i=1}^n \gamma^*_i \phi(x_i),
\end{equation*} 
  where $\gamma^*_i$ are the weights of 
  the spatial median for the modified kernel 
  matrix. 
\vspace{0.5em}	
\item The kernel Mahalanobis distance is 
  calculated as
\begin{align} \label{eq:refin_dist}		
	d^{*}_\mF(x) 
	&= (\phi(x) - c^*_\mF)^\top \, 
	   (\Sigma^{*}_\mF)^{-1} \, 
		 (\phi(x) - c^*_\mF) \nonumber \\ 
	&= (\phi(x) - c^*_\mF)^\top \, 
	   \tilde{\Phi}^\top D^{\frac{1}{2}}V 
		 L^{-1} V^\top D^{\frac{1}{2}}
		 \tilde{\Phi} \, 
		 (\phi(x) - c^*_\mF) \nonumber \\
	&= k^*(x,X) D^{\frac{1}{2}} V L^{-1} 
	   V^\top D^{\frac{1}{2}}{k^*(x,X)}^\top
\end{align}
	with
\begin{equation*}
\begin{aligned}
	k^*(x,X) &= k(x,X) - \sum_{i=1}^{n} 
	  w_i k(x,x_i) 1_n^\top \\ 
	&\quad - \sum_{j=1}^{n} 
	  \gamma_j^* k(x_j,X) - \sum_{i=1}^{n} 
		\sum_{j=1}^{n} w_i 
		\gamma_j^* k(x_i,x_j) 1_n^\top
\end{aligned}
\end{equation*}
	where $k(x,X) = [k(x,x_1),\ldots,k(x,x_n)]$.
\end{enumerate}
The $h$ points with the smallest 
$d^{*}_\mF(x)$ form the refined $h$-subset. 
The entire procedure is Algorithm 6 
in the Supplementary Material.

\subsection{Kernel MRCD algorithm} 
\label{subsec:kmcdalgorithm}

We now have all the elements to compute the 
kernel MRCD by Algorithm \ref{algo:rKMRCD}. 
Given any PSD kernel matrix and subset size 
$h$, the algorithm starts by computing the
four initial estimators described in 
Section \ref{subsec:initialestimates}. 
Each initial estimate is then refined 
according to Section \ref{subsec:refinement}. 
Next, kernel MRCD computes the 
regularization parameter $\rho$. 
This is done with a kernelized version of 
the procedure in \cite{boudt2020minimum}.
For each initial estimate we choose $\rho$ 
such that the regularized kernel matrix 
$\tilde{K}^H_\mathrm{reg}$ of
\eqref{eq:egKernelMatrix} is 
well-conditioned.
If we denote by $\lambda$ the vector 
containing the eigenvalues of the centered 
kernel matrix $\tilde{K}^H$, the 
condition number of 
$\tilde{K}^H_\mathrm{reg}$ is
\begin{equation}
\label{eq:conditionNumber}
	\kappa(\rho) = \frac{(h-1)\rho +
	(1-\rho)\max(\lambda)}
	{(h-1)\rho + (1-\rho)\min(\lambda)}
\end{equation}
and we choose $\rho$ such that 
$\kappa(\rho) \leqslant 50$.
\blue{(Section A.3 in the supplementary
material contains a simulation study
supporting this choice.)}
Finally, kernel C-steps are applied until
convergence, where we monitor the 
objective function of Section 
\ref{subsec:kmcdintro}. 

\begin{algorithm}[!ht]
\caption{Kernel MRCD.}
\label{algo:rKMRCD}
\begin{enumerate}
\item \textbf{Input}: kernel matrix $K$,  
  subset size $h$.
\item Compute the weights of the four
  initial estimates of location and 
	scatter as in Section 
	\ref{subsec:initialestimates}.
\item Refine each initial estimate as
  in Section 
	\ref{subsec:refinement}.
\item For each resulting subset, 
  determine $\rho^{(i)}$ such that 
	$\kappa(\rho^{(i)}) \leqslant 50$.
\item Determine the final $\rho$ as
 in \cite{boudt2020minimum}: if 
  $\max_i \rho^{(i)} \leqslant 0.1$
	set $\rho = \max_i \rho^{(i)}$, 
	otherwise set $\rho = \max \left
	(0.1,\med_i(\rho^{(i)})\right)$.
\item For $H = H^{(1)}, \dots, H^{(4)}$ 
  perform C-steps as follows:	
  \begin{enumerate}
	\item Compute the regularized kernel 
	  matrix $\tilde{K}^H_\mathrm{reg}$ for 
		the $h$-subset $H$ from
		\eqref{eq:egKernelMatrix}.
	\item Calculate the regularized 
	  Mahalanobis distance for each 
		observation	$i$ by \eqref{eq:regMahal}.
	\item Redefine $H$ as the $h$ indices 
		$i$ with smallest distance.
	\item Compute and store the objective. 
	  If not converged, go back to (a).  
	\end{enumerate}
\item Select the $h$-subset with the 
  overall smallest objective function.
\item \textbf{Output}: the final 
  $h$-subset and the robust distances.
\end{enumerate}
\end{algorithm} 

In the special case where the linear
kernel is used, the centered kernel matrix 
$\tilde{K}^H$ immediately yields the 
regularized covariance matrix 
$\hSigma_\mathrm{reg}^H$ through
\begin{equation*}\label{eq:sigmaregfromh}
	\hSigma_\mathrm{reg}^H = 
	\frac{1-\rho}{h-1}(\tilde{X}^H)^\top
	\tilde{V} \Lambda \tilde{V}^\top 
	\tilde{X}^H + \rho I_h
\end{equation*}
where $\tilde{X}^H = X^H - \frac{1}{h}
\sum_{i \in H} x_i$ is the centered 
matrix of the observations in $H$ and 
$\Lambda$ and $\tilde{V}$ contain the 
eigenvalues and normalized eigenvectors
of $\tilde{K}^H$. 
(The derivation is given in Section A.2.)
So instead of applying MRCD to 
coordinate data we can also run KMRCD 
with a linear kernel and transform 
$\tilde{K}^H$ to 
$\hSigma_\mathrm{reg}^H$ afterward.
This computation is faster when the
data has more dimensions than cases.

\subsection{Anomaly detection by KMRCD}
\label{subsec:anomaly}

Mahalanobis distances (MD) relative to 
robust estimates of location and scatter 
are very useful to flag outliers, because
outlying points $i$ tend to have higher
$\MD_i$ values.
The standard way to detect outliers by 
means of the MCD in low dimensional data 
is to compare the robust distances to a 
cutoff that is the square root of a 
quantile of the chi-squared distribution 
with degrees of freedom equal to the 
data dimension
\cite{rousseeuw1999fast}.
However, in high dimensions the distribution
of the squared robust distances is no longer
approximately chi-squared, which makes it
harder to determine a suitable cutoff value.
Faced with a similar problem 
\cite{rousseeuw2018measure} introduced a 
different approach, based on the empirical
observation that robust distances of the 
non-outliers in higher dimensional data 
tend to have a distribution that is 
roughly similar to 
a lognormal.
They first transform the distances
$\MD_i$ to $\LD_i = \log(0.1 + \MD_i)$,
where the term $0.1$ prevents numerical
problems should a (near-)zero $\MD_i$ occur.
The location and spread of the non-outlying 
$\LD_i$ are then estimated by 
$\hmu_\mathrm{MCD}$ and 
$\hsigma_\mathrm{MCD}$, the results of 
applying the univariate MCD to all $\LD_i$
using the same $h$ as in the KMCRD method
itself.
Data point $i$ is then flagged iff
\begin{equation*}
	\frac{\LD_{i}-\hmu_\mathrm{MCD}(\LD)}
	{\hsigma_\mathrm{MCD}(\LD)}
	> z(0.995)
	\end{equation*}
where $z(0.995)$ is the 0.995 quantile of
the standard normal distribution.
The cutoff value for the untransformed
robust distances is thus
\begin{equation} \label{eq:finalCutoff}
	c=\exp \left(\hmu_\mathrm{MCD}(\LD)+
	  z(0.995)\hsigma_\mathrm{MCD}(\LD) 
		\right)-0.1\;. 
\end{equation}
The user may want to try different values 
of $h$ to be used in both the KMRCD method
itself as well as in the $\hmu_\mathrm{MCD}$ 
and $\hsigma_\mathrm{MCD}$ in
\eqref{eq:finalCutoff}.
One typically starts with a rather low
value of $h$, say $h=0.5n$ when the linear
kernel is used and there are up to 10
dimensions, and $h=0.75n$ in all other
situations.
This will provide an idea about the number
of outliers in the data, after which it is 
recommended to choose $h$ as high as 
possible provided $n-h$ exceeds the number
of outliers.
This will improve the accuracy of the
estimates.

\subsection{Choice of bandwidth}
\label{sec:bandwidth}	
	 
A commonly used kernel function is the 
radial basis function (RBF) 
$k(x,y) = e^{-\|x-y\|^2/(2\sigma^2)}$
which contains a tuning constant
$\sigma$ that needs to be chosen.
When the downstream learning task is
classification $\sigma$ is commonly 
selected by cross validation, where it 
is assumed that the data has no outliers
or they have already been removed.
However, in our unsupervised outlier 
detection context there is nothing to
cross validate. 
Therefore, we will use the so-called
median heuristic \cite{gretton2012kernel}
given by
\begin{equation} \label{eq:heuristic}
  \sigma^2 = \med \{
	\|x_{i}-x_{j}\|^2\;;\; 
	1 \leqslant i<j \leqslant n\}
\end{equation}
in which the $x_i$ are the standardized
data in the original space.
We will use this $\sigma$ in all our 
examples.

\subsection{Illustration on toy examples} 	
\label{sec:toy}	
	
We illustrate the proposed KMRCD 
method on the two toy examples in 
Figure \ref{fig:toy1}.
Both datasets consist of $n=1000$ 
bivariate observations.
The elliptical dataset in the left panel 
was generated from a bivariate Gaussian
distribution, plus 20\% of outliers.
The non-elliptical dataset in the panel on
the right is frequently used to demonstrate 
kernel methods \cite{suykens2002least}. 
This dataset also contains $20\%$ of 
outliers, which are shown in red and form 
the outer curved shape.
We first apply the non-kernel MCD method,
which does not transform the data, with
$h=\floor{0.75n}$. (Not using a kernel is
equivalent to using the linear kernel.)
The results are in 
Figure \ref{fig:toylinear}.
In the panel on the left this works well 
because the MCD method was developed for 
data of which the majority has a roughly 
elliptical shape.
For the same reason it does not work well 
on the non-elliptical data in the right
hand panel.
\begin{figure}[ht]
\centering
\includegraphics[width=.45\textwidth]
  {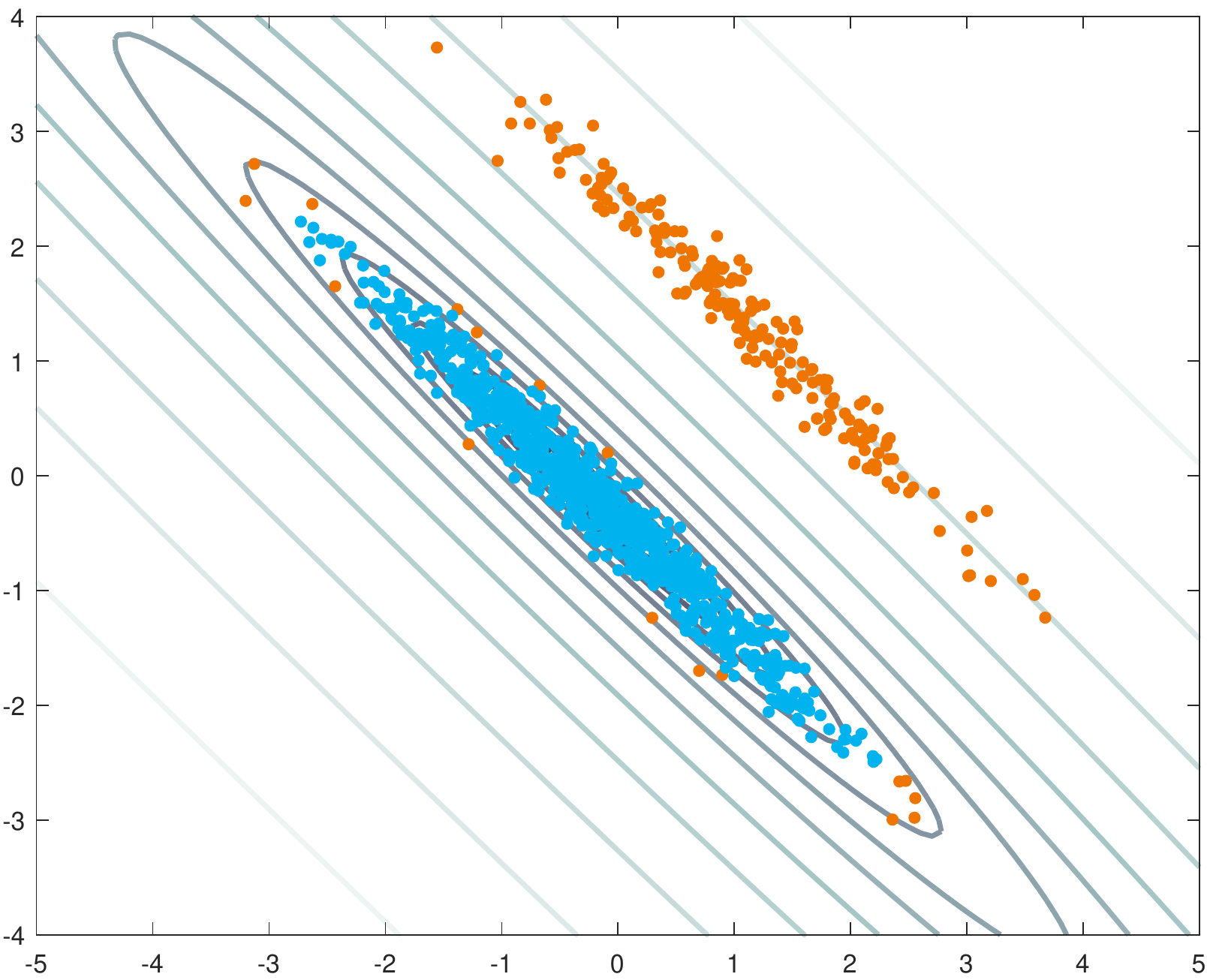}
\hspace{0.5cm}
\includegraphics[width=.45\textwidth]
  {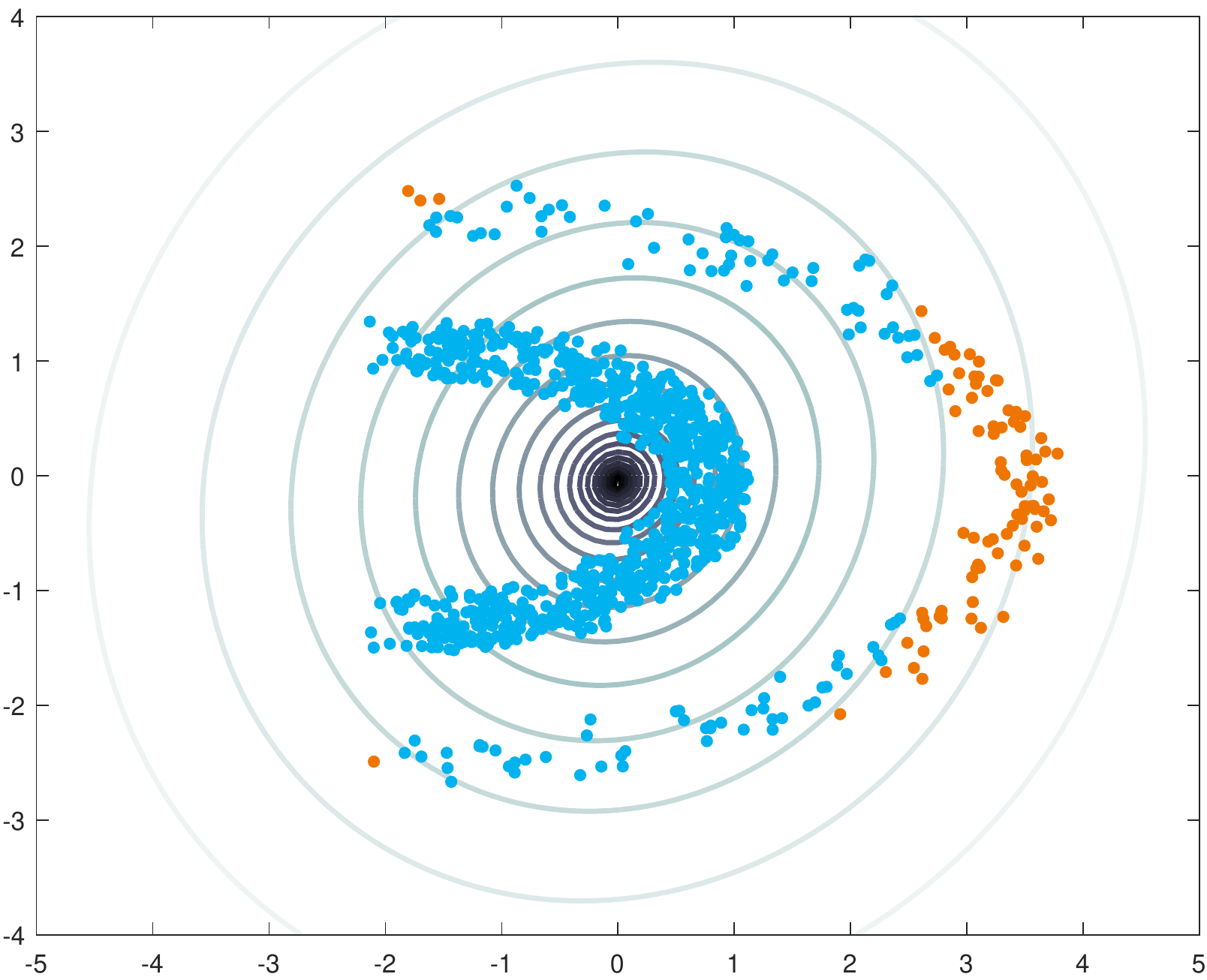}
\caption{Results of the non-kernel MCD 
method on the toy datasets of Figure 
\ref{fig:toy1}. The contour lines are 
level curves of the MCD-based 
Mahalanobis distance.}
\label{fig:toylinear}
\end{figure} 

We now apply the kernel MRCD method to
the same datasets.
For the elliptical dataset we use the
linear kernel, and for the non-elliptical
dataset we use the RBF kernel with tuning
constant $\sigma$ given by formula
\eqref{eq:heuristic}.
This yields Figure~\ref{fig:toyKMRCD}.
We first focus on the left hand column.
The figure shows three stages of the
KMRCD runs. At the top, in 
Figure~\ref{fig:toyKMRCD}(a), 
we see the 
result for the selected $h$-subset, 
after the C-steps have converged.
The members of that $h$-subset are the
green points, whereas the points 
generated as outliers are colored red.
Since $h$ is lower than the true number
of inlying points, some inliers (shown
in black) are not included in the 
$h$-subset.
In the next step, 
Figure~\ref{fig:toyKMRCD}(b) shows the
robust Mahalanobis distances, with
the horizontal line at the 
cutoff value given by formula
\eqref{eq:finalCutoff}.
The final output of KMRCD shown in
Figure~\ref{fig:toyKMRCD}(c) has the
flagged outliers in orange and the
points considered inliers in blue.
As expected, this result is similar 
to that of the non-kernel MCD in the
left panel of Figure 
\ref{fig:toylinear}.

The right hand column of 
Figure~\ref{fig:toyKMRCD} shows the
stages of the KMRCD run on the 
non-elliptical data. The results for
the selected $h$-subset in 
Figure~\ref{fig:toyKMRCD}(a) look much
better than in the right hand panel of
Figure \ref{fig:toylinear}, because 
the level curves of the robust 
distance now follow the shape of the 
data. In stage (b) we see that the
distances of the inliers and the 
outliers are fairly well separated by
the cutoff \eqref{eq:finalCutoff},
with a few borderline cases, and
stage (c) is the final result.
This illustrates that using a nonlinear
kernel allows us to fit non-elliptical
data.

\begin{figure}[ht]
\centering
\includegraphics[width=0.95\textwidth]
  {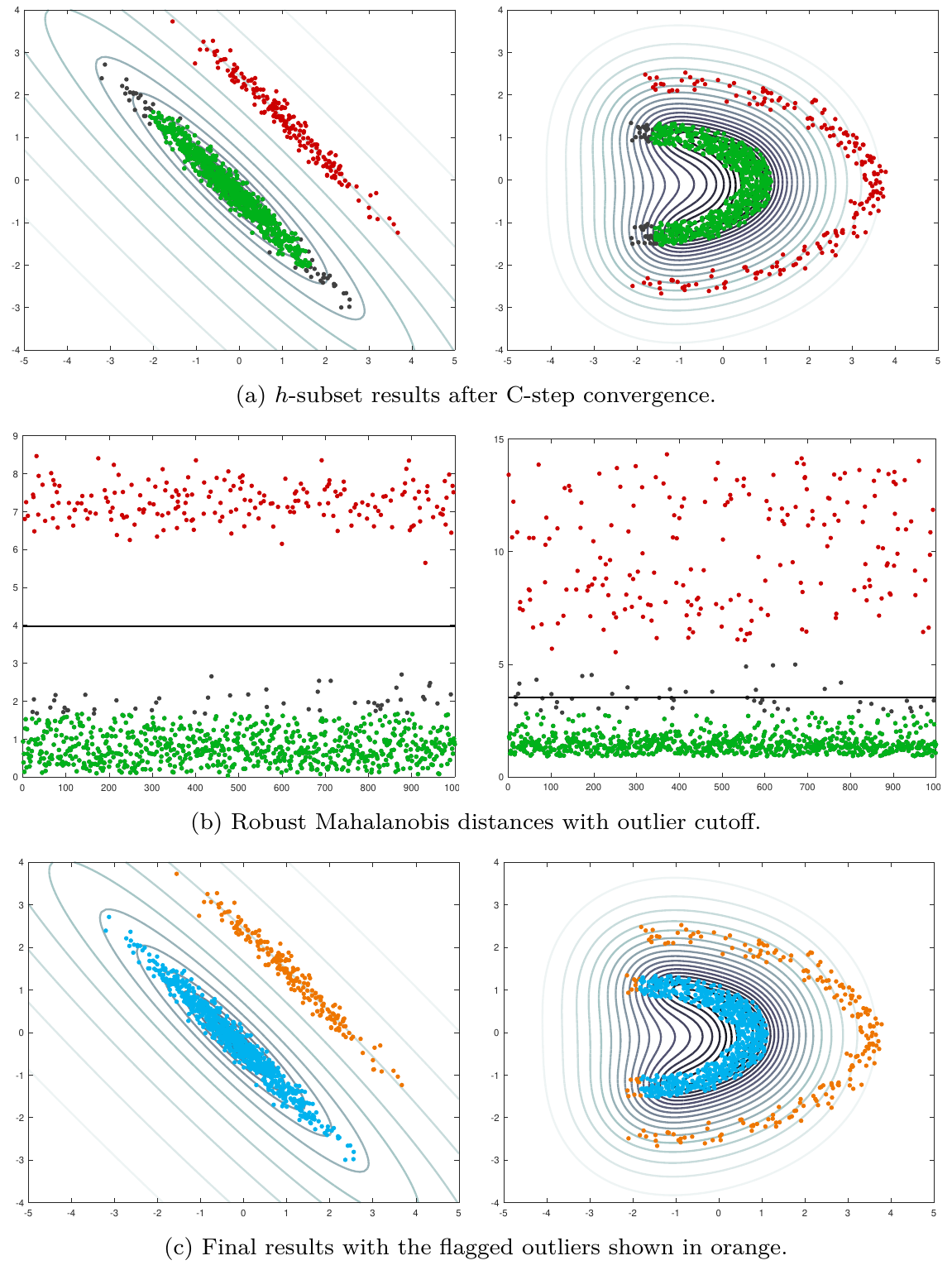}
\caption[]{Kernel MRCD results on the 
  toy datasets of Figure \ref{fig:toy1}. 
	In the left column the linear kernel 
	was used, and in the right column the 
	RBF kernel. The three stages (a), (b)
	and (c) are explained in the text.}
\label{fig:toyKMRCD}
\end{figure}
\FloatBarrier

\section{Simulation Study}
\label{sec:simulation}	

\subsection{Simulation study with 
  linear kernel}	
\label{subsec:Lin_simulation}	

In this section we compare the KMRCD 
method proposed in the current paper,
run with the linear kernel, to the 
MRCD estimator of Boudt et al. 
\cite{boudt2020minimum}.
\blue{Recall that using the linear 
kernel $k(x,y) = x^\top y$ means that 
the feature space can be taken identical 
to the coordinate space, so using the 
linear kernel is equivalent to not
using a kernel at all.}
Our purpose is twofold.
First, we want to verify whether KMRCD 
performs well in terms of robustness
and accuracy, and that its results are 
consistent with those of MRCD.
And secondly, we wish to measure the
computational speedup obtained by KMRCD 
in high dimensions.
In order to obtain a fair comparison
we run MRCD with the identity
matrix as target, which corresponds 
to the target of KMRCD.
All computations are done in 
MATLAB on a machine with Intel 
Core i7-8700K and $16$ GB of 
$3.70$GHz RAM. 
	
For the uncontaminated data, that is,
for contamination fraction 
$\varepsilon = 0$, we generate 
$n$ cases from a $p$-variate normal 
distribution with true covariance
matrix $\Sigma$. 
Since the methods under consideration 
are equivariant under translations and
rescaling variables, we can assume
without loss of generality that the 
center $\mu$ of the distribution is 0 
and that the diagonal elements of 
$\Sigma$ are all 1. We denote the
distribution of the clean data by
$\mathcal{N}(0,\Sigma)$. 
Since the methods are not equivariant 
to arbitrary nonsingular affine
transformations we cannot set $\Sigma$ 
equal to the identity matrix. 
Instead we consider $\Sigma$ of the 
ALYZ type, generated as in Section 4 
of \cite{agostinelli2015robust}, which 
yields a different $\Sigma$ in each 
replication, \blue{but always with
condition number $100$. The main 
steps of the construction of $\Sigma$ 
in \cite{agostinelli2015robust} are
the generation of a random orthogonal
matrix to provide eigenvectors, and
the generation of $p$ eigenvalues
such that the ratio between the largest
and the smallest is 100, followed by
iterations to turn the resulting
covariance matrix into a correlation
matrix while preserving the
condition number. (In section A.3 of
the supplementary material also
$\Sigma$ matrices with higher condition
numbers were generated, with similar
results.)}

For a contamination fraction
$\varepsilon > 0$ we replace a random
subset of $\floor{\varepsilon n}$ 
observations by outliers of different 
types.
{\it Shift contamination} is generated 
from $\mathcal{N}(\mu_C, \Sigma)$ where 
$\mu_C$ lies in the direction where the 
outliers are hardest to detect, which is
that of the last eigenvector $v$ of the 
true covariance matrix $\Sigma$. 
We rescale $v$ by making 
$v^T \Sigma^{-1} v = E[Y^2] = p$ where 
$Y^2 \sim \chi^2_p$\;.
The center is taken as $\mu_C = k v$ 
where we set $k=200$. 
Next, {\it cluster contamination} stems 
from $\mathcal{N}(\mu_C, 0.05^2\,I_p)$ 
where $I_p$ is the identity matrix. 
Finally, {\it point contamination} 
places all outliers in the point $\mu_C$ 
so they behave like a tight cluster. 
These settings make the simulation 
consistent with those in 
\cite{boudt2020minimum,hubert2012deterministic} 
and \cite{DeKetelaere:RT-DetMCD}.
The deviation of an estimated scatter
matrix $\hSigma$ relative to the true 
covariance matrix $\Sigma$ is measured 
by the Kullback–Leibler (KL) divergence 
$\mbox{KL}(\hSigma,\Sigma) = \trace(
 \hSigma \Sigma^{-1}) - \log(\det(
 \hSigma \Sigma^{-1})) - p$.
The speedup factor is measured as
$\mbox{speedup}=\mbox{time}(\mbox{MRCD})/
 \mbox{time}(\mbox{KMRCD})$.
Different combinations of $n$ and $p$ are 
generated, ranging from $p = n/2$ to
$p = 2n$.

Table \ref{tbl:KL} presents the 
Kullback–Leibler deviation results. 
The top panel is for $\eps=0$, the 
middle panel for $\eps=0.1$ and the 
bottom panel for $\eps=0.3$\,.
All table entries are averages over 50 
replications.  
First look at the results without
contamination. 
By comparing the three choices for 
$h$, namely $\floor{0.5n}$, 
$\floor{0.75n}$ and $\floor{0.9n}$,
we see that lowering $h$ in this
setting leads to increasingly
inaccurate estimates $\hSigma$.
This is the price we pay for being
more robust to outliers, since
$n-h$ is an upper bound on the number
of outliers the methods can handle.
When we look at the panels for
higher $\varepsilon$ we see a similar
pattern. When $\varepsilon = 0.1$
the choice $\floor{0.9n}$ is 
sufficiently robust, and the lower
choices of $h$ have higher KL
deviation. But when 
$\varepsilon = 0.3$ only the choice 
$h = \floor{0.5n}$ can detect the 
outliers, the other choices cause
the estimates to break down.
These patterns are confirmed by the 
averaged 
$\mbox{MSE}=\sum_{i=1}^{p}\sum_{j=1}^{p} 
 (\hSigma - \Sigma)^2_{ij}/p^2$
shown in Table 7 in 
the Supplementary Material. 

From these results we conclude that
it is important that $h$ be chosen
lower than $n$ minus the number of
outliers, but not much lower since
that would make the estimates less
accurate. A good strategy is to 
first run with a low $h$, which 
reveals the number of outliers, and
then to choose a higher $h$ that can 
still handle the outliers and yields
more accurate results as well.
 
As expected the KMRCD results are 
similar to those of MRCD, but
not identical because there are
differences in the selection of 
initial estimators, also leading
to differences in the resulting
regularization parameter $\rho$
shown in Table 8 
in the Supplementary Material.

\begin{table*} 
\caption{Kullback-Leibler deviations 
         of $\hSigma$ from $\Sigma$. } 
\label{tbl:KL}
\centering               
\begin{adjustbox}{width=0.85\textwidth}                        
\centering
\begin{tabular}{cccccccccc}
\toprule                 
	& \multicolumn{3}{c}{Point contamination} & \multicolumn{3}{c}{Shift contamination}  & \multicolumn{3}{c}{Cluster contamination}\\ 
	& \multicolumn{3}{c}{Value of $h/n$} & \multicolumn{3}{c}{Value of $h/n$}  & \multicolumn{3}{c}{Value of $h/n$}\\ 	
				\cmidrule(lr){2-4} \cmidrule(lr){5-7} \cmidrule(lr){8-10}
 & 0.50 & 0.75 & 0.90 & 0.50 & 0.75 & 0.90 & 0.50 & 0.75 & 0.90\\
\textbf{$\varepsilon=0$\,:} \\
\midrule
\textbf{KMRCD} \\
	400$\times$200	&126.72		&80.03		&64.65		&127.66		&78.99	&64.54		&129.37		&79.42		&64.45			\\
	300$\times$200	&174.37		&110.03		&88.43		&176.46		&108.45		&87.94		&174.52		&109.68		&87.50		\\
	200$\times$200	&262.41		&171.74		&140.18		&263.03		&172.23		&140.66		&260.21		&169.02		&140.45			\\
	200$\times$300	&492.70		&381.13		&319.42		&491.65		&379.07		&317.45		&491.64		&373.38		&319.44			\\
	200$\times$400	&724.41		&602.78		&535.59		&715.59		&602.44		&532.55		&731.76		&607.21		&537.27			\\ 
\; \\
\textbf{MRCD} \\
	400$\times$200	& 126.58 & 80.82	&65.65	&127.00	&79.99 &65.63 & 128.92	& 80.63	& 65.35	\\
	300$\times$200	&175.88	&110.57	&89.49	&176.21	&109.22	&88.79	&174.36		&110.28		&88.85			\\
	200$\times$200	&265.04		&174.13	&141.23	&264.57	&173.56	&141.10	&261.93		&172.00	&141.05			\\
	200$\times$300	&499.11		&384.91		&323.48	&500.80			&383.02		&322.14		&499.30	&378.74		&324.27	\\
	200$\times$400	&734.47		&608.04		&539.54		&729.52		&610.79		&539.24		&738.21		&611.83		&543.84	\\
\; \\
\textbf{$\varepsilon=0.1$\,:} \\
\midrule
\textbf{KMRCD} \\
	400$\times$200	&128.14	&78.48 &63.16	&127.28	&78.28	&62.28	&128.91	&79.75	&62.71		\\
	300$\times$200	&176.76	&107.10	&86.94	&174.32	&109.86	&87.37	&176.13	&108.54	&87.86		\\
	200$\times$200	&263.76	&172.21	&137.31	&260.06	&171.46	&137.72	&260.42	&171.76	&136.27		\\
	200$\times$300	&493.36	&368.80	&311.48	&488.10		&377.46	&311.04	&491.24	&378.81	&319.45		\\
	200$\times$400	&728.07	&600.12	&558.79	&723.40		&596.87	&535.08	&720.44	&604.06	&534.81		\\
\; \\
\textbf{MRCD} \\
	400$\times$200	&128.22	&79.59	&64.19	&127.92	&79.76	&63.48	&129.09	&81.11	&63.91		\\
	300$\times$200	&174.17	&107.67	&88.05	&172.96		&111.15	&88.00	&173.46	&109.50	&89.12		\\
	200$\times$200	&262.71	&171.55	&137.51	&259.52	&170.79	&138.03	&261.48	&170.95	&136.31		\\
	200$\times$300	&493.66	&368.88	&309.88	&494.44	&382.69	&312.85	&499.27	&382.10	&320.48		\\
	200$\times$400	&723.42	&599.88	&525.99	&736.52	&601.15	&536.73	&733.36	&611.43	&537.56		\\
\; \\
\textbf{$\varepsilon=0.3$\,:} \\
\midrule
\textbf{KMRCD} \\
	400$\times$200	&127.46	&4914.7	&2073.9	&126.26	&1142.7	&1613.4	&124.73	&1124.8	&1600.6 \\
	300$\times$200	&176.51	&5104.1	&2046.1	&176.82	&1125.9	&1597.6	&173.30	&1117.2	&1555.3 \\
	200$\times$200	&257.91	&5180.6	&2038.3	&255.55	&1168.8	&1559.1	&257.90	&1163.6	&1535.9 \\
	200$\times$300	&485.71	&5494.5	&2230.0	&488.19	&1310.0	&1626.8	&490.05	&1311.1	&1616.9 \\
	200$\times$400	&714.57	&5779.1	&2316.1	&721.41	&1448.7	&1736.7	&718.17	&1423.4	&1721.1 \\
\; \\
\textbf{MRCD} \\
	400$\times$200	&124.33	&6771.6	&3082.7	&125.15	&1395.6	&2068.9	&124.68	&1371.1	&2078.5 \\
	300$\times$200	&164.89	&7118.6	&3049.4	&172.11	&1415.2	&2076.5	&168.51	&1393.0	&2011.1 \\
	200$\times$200	&237.08	&7519.3	&3075.1	&241.27	&1481.1	&2040.1	&242.59	&1485.8	&2014.7 \\
	200$\times$300	&450.12	&8233.1	&3413.4	&483.07	&1653.4	&2122.5	&483.67	&1659.9	&2102.4 \\
	200$\times$400	&663.35	&8585.0	&3507.7	&717.65	&1812.8	&2212.9	&719.74	&1790.5	&2201.6 \\
\bottomrule
\end{tabular}
\end{adjustbox}
\end{table*}
	
\begin{table*} 
\caption{Speedup factors of KMRCD relative to MRCD.} 
\label{tbl:speedups}
\centering               
\begin{adjustbox}{width=0.9\textwidth}                        
\centering
\begin{tabular}{cccccccccc}
\toprule                 
  & \multicolumn{3}{c}{Point contamination} & \multicolumn{3}{c}{Shift contamination}  & \multicolumn{3}{c}{Cluster contamination}\\ 
	& \multicolumn{3}{c}{Value of $h/n$} & \multicolumn{3}{c}{Value of $h/n$}  & \multicolumn{3}{c}{Value of $h/n$}\\ 	
				\cmidrule(lr){2-4} \cmidrule(lr){5-7} \cmidrule(lr){8-10}
 & 0.50 & 0.75 & 0.90 & 0.50 & 0.75 & 0.90 & 0.50 & 0.75 & 0.90\\
\textbf{$\varepsilon=0$\,:} \\
\midrule
 400$\times$200	&97	&95	&89	&99	&96	&91	&100&98	&92 \\
 300$\times$200	&281	&242	&229	&281	&242	&229	&284	&242	&228\\
 200$\times$200	&301	&249	&227	&299	&253	&226	&301	&253	&228\\
 200$\times$300	&661	&558	&516	&623	&562	&519	&657	&563	&520\\
 200$\times$400	&1144	&979	&897	&1157	&982	&892	&1159	&978	&899\\
\; \\
\textbf{$\varepsilon=0.1$\,:} \\
\midrule
 400$\times$200	&98	&90	&87	&96	&92	&89	&96	&92	&90	\\	
 300$\times$200	&263	&227	&211	&278	&240	&225	&282	&239	&224	\\
 200$\times$200	&292	&243	&216	&302	&252	&225	&302	&249	&228	\\
 200$\times$300	&631	&534	&504	&652	&564	&516	&664	&567	&512 \\
 200$\times$400	&1113	&951	&870	&1157	&981	&902	&1110	&980	&903 \\
\; \\
\textbf{$\varepsilon=0.3$\,:} \\
\midrule
 400$\times$200	&77	&79	&72	&100	&96	&93	&99	&95	&93 \\
 300$\times$200	&211	&193	&185	&281		&251	&233	&288	&251	&234 \\
 200$\times$200	&234	&206	&202	&301	&262	&238	&299	&257	&238 \\
 200$\times$300	&543	&472	&432	&653	&564	&522	&654	&566&520 \\
 200$\times$400	&1000	&791	&749	&1161	&976	&911	&1150	&977&900 \\
\bottomrule
\end{tabular}
\end{adjustbox}
\end{table*}

We now turn our attention to the 
computational speedup factors in 
Table \ref{tbl:speedups}, that were
derived from the same simulation
runs as Table \ref{tbl:KL}. Overall 
KMRCD ran substantially faster than
MRCD, with the factor becoming
larger when $n$ decreases and/or
the dimension $p$ increases.
There are two reasons for the
speedup. First of all, the MRCD 
algorithm computes six initial 
scatter estimates, of which the last 
one is the most computationally
demanding since it computes a robust 
bivariate correlation of every pair
of variables, requiring $p(p-1)/2$
computations whose total time 
increases fast with $p$. 
Part of the speedup stems
from the fact that KMRCD does not
use this initial estimator, whereas
its own four kernelized initial 
estimates gave equally robust results.
This explains most of the speedup
in Table \ref{tbl:speedups}.

For $p>n$ there is a second reason
for the speedup, the use of the kernel 
trick. 
\blue{In particular, each C-step 
requires the computation of the 
Mahalanobis distances of all cases.
MRCD does this by inverting the
$p \times p$ covariance matrix 
$\hSigma^H_\mathrm{reg}$\,, whereas
KMRCD uses 
equation~\eqref{eq:regMahal} which
implies that it suffices to invert 
the $n \times n$ kernel matrix
$\tilde{K}_{\mathrm{reg}}^H$\,, 
which takes time complexity 
$O(n^3)$ instead of $O(p^3)$.}

\subsection{\blue{Simulation with nonlinear 
                  kernel}}
\label{subsec:NonLin_simulation}	

\blue{In this section we compare the proposed 
KMRCD estimator to the MRCD estimator of Boudt 
et al.~\cite{boudt2020minimum} on two types
of non-elliptical datasets.
The first type is generated by a copula.
We start by considering the $t$
copula~\cite{nelsen2007introduction} with
Pearson correlation 0.1 
and $\nu = 1$ degrees of freedom.
The black points in the left panel of
Figure~\ref{fig:nonLinSimulation_T} were
generated from this copula.
We then added contamination in the form 
of uniformly distributed random noise
on the unit square, where points lying 
close to the regular distribution were 
removed. 
The red points in the left panel of
Figure~\ref{fig:nonLinSimulation_T} are
the outliers.
Apart from the $t$ copula we also consider 
the Frank, Clayton, and Gumbel copulas with
Kendall rank correlation 
$\tau = 0.6$\,. 
They are visualized in Figure~9
in section A.6 of the Supplementary 
Material.}
 
\begin{figure}[ht]
\centering
\includegraphics[width=1.0\textwidth]
	{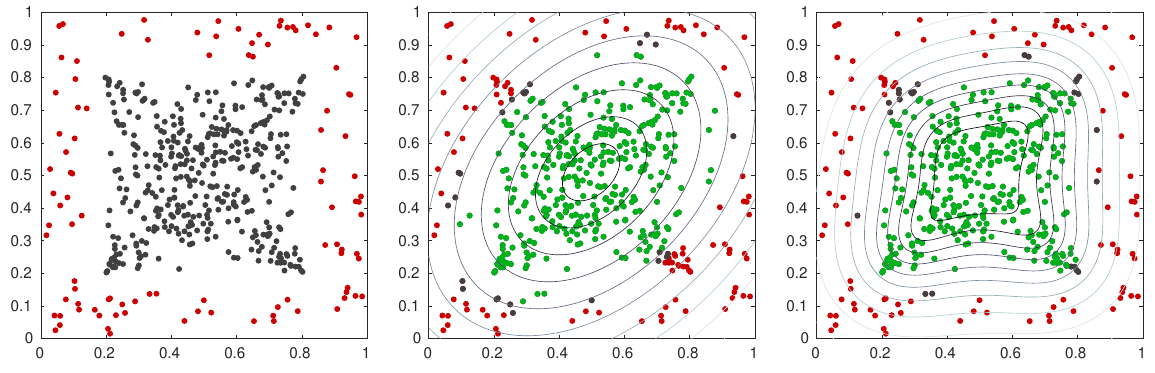}	
\caption{\blue{Illustration of the 
non-elliptical 
simulation setting with data generated from 
the $t$ copula, plus $20\%$ of outlying 
observations. In the left panel, the regular 
observations are shown in black and the outliers 
in red. The results of the MRCD estimator are in 
the middle panel, and those of the KMRCD 
estimator in the rightmost panel, each for
$h = 0.75n$. In those panels the points in the 
$h$-subset are shown in green, and the other
points with the $n(1-\varepsilon)$ lowest 
(kernel) Mahalanobis distance are depicted in 
grey. The remaining points are shown in red. 
The curves are contours of the 
robust (kernel) Mahalanobis distance.}}
\label{fig:nonLinSimulation_T}
\end{figure}

\blue{The proposed estimator is also 
benchmarked in a second type of setting
where the regular observations are 
uniformly distributed on the unit circle, 
and inside the circle are outliers generated 
from the Gaussian distribution with 
center $0$ and covariance matrix equal
to $0.04$ times the identity matrix.
This setting is illustrated in 
Figure~\ref{fig:nonLinSimulation_Circle}.
This is a simple example where the clean
data lie near a manifold.} 

\begin{figure}[ht]
\centering
\includegraphics[width=1.0\textwidth]
	{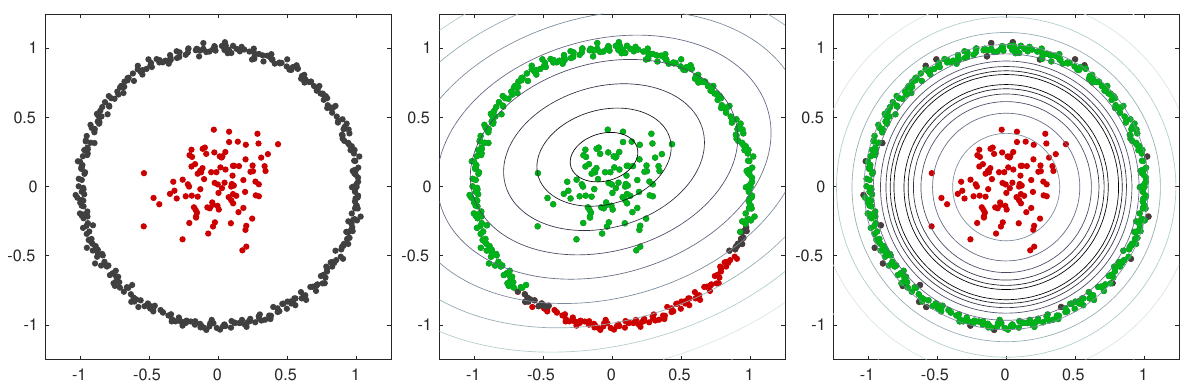}	
\caption{\blue{Illustration of the 
non-elliptical 
simulation setting with data generated from 
the circle manifold, plus $20\%$ of outlying 
observations. The remainder of the description
is as in Figure~\ref{fig:nonLinSimulation_T}.}}
\label{fig:nonLinSimulation_Circle}
\end{figure}

\blue{In the simulation we generated 100
datasets of each type, with $n=500$ and
the outlier fraction $\varepsilon$ equal
to 0.1 or 0.2, so the number of regular
observations is $n(1-\varepsilon)$.  
With all four copulas the KMRCD estimator 
used the radial basis function with
bandwidth~\eqref{eq:heuristic}.
For the circle-based data the polynomial
kernel $k(x,y) = (x^\top y + 1)^2$ of
degree 2 was used.}

\blue{We measure the performance by counting 
the number of outliers in the $h$-subset, and 
among the $n(1-\varepsilon)$ points with the 
lowest (kernel) Mahalanobis distance. 
The averaged counts over the 100 replications
are shown in Table~\ref{tbl:nonLin}.
By comparing the rows of KMRCD and MRCD with
the same $\varepsilon$, we see that MRCD has 
more true outliers in its $h$-subset and its
$n(1-\varepsilon)$ set.
In the table, KMRCD outperforms MRCD for both 
choices of $\varepsilon$ and for all three 
choices of $h$.
The good performance of KMRCD is also seen 
in the right panel of 
Figure~\ref{fig:nonLinSimulation_T}, where 
the contours of the kernel Mahalanobis 
distance nicely follow the distribution. 
The difference between MRCD and KMRCD is 
most apparent on the circle-based data: in 
Figure~\ref{fig:nonLinSimulation_Circle} the
KMRCD fits the regular data on the 
circle, whereas the original MRCD method, 
by its nature, considers the outliers in the
center as regular data.}

\begin{table*} 
\caption{\blue{Average number of outliers in 
the $h$-subset $H$, and among the 
$n(1-\varepsilon)$ points with lowest 
(kernel) Mahalanobis distance.} } 
\label{tbl:nonLin}
\centering               
\begin{adjustbox}{width=1\textwidth}                        
\centering
\begin{tabular}{cccccccccccccccc}
\toprule                 
	& \multicolumn{3}{c}{$t$ copula} 
	& \multicolumn{3}{c}{Frank copula} 
	& \multicolumn{3}{c}{Clayton copula}
	& \multicolumn{3}{c}{Gumbel copula}
	& \multicolumn{3}{c}{Circle}\\ 
\cmidrule(lr){2-4} \cmidrule(lr){5-7} \cmidrule(lr){8-10} \cmidrule(lr){11-13} \cmidrule(lr){14-16}
	& 0.75 & 0.8 & 0.9 & 0.75 & 0.8 & 0.9 
	& 0.75 & 0.8 & 0.9 & 0.75 & 0.8 & 0.9 
	& 0.75 & 0.8 & 0.9 \\
\textbf{$\varepsilon=0.1$\,:} \\
\midrule
\textbf{KMRCD} \\
	$H$	& 0 & 0 & 0.5 & 0 & 0 & 2.9 & 0 & 0 & 2.6 & 0 & 0 & 2.8 & 0 & 0 & 0	\\
	$n(1-\varepsilon)$ & 10 & 8.2 & 0.5 & 3.9 & 4.2 & 2.9 & 4.3 & 4.4 & 2.6 & 4.4 & 4.7 & 2.8 & 0 & 0 & 0	\\
 \; \\
\textbf{MRCD} \\
 $H$ & 1.2 & 2.9 & 13.3 & 1.3 & 2.1 & 11.2 & 2.4 & 3.7 & 11.6 & 2.0 & 2.9 & 11.2 & 50 & 50 & 50	\\
	$n(1-\varepsilon)$ & 22.5 & 20.3 & 13.3 & 12.9 & 12.2 & 11.2 & 14.4 & 14.0 & 11.6 & 13.7 & 13.3 & 11.2 & 50 & 50 & 50	\\
 \; \\
\textbf{$\varepsilon=0.2$\,:} \\
\midrule
\textbf{KMRCD} \\
 $H$ & 0.1 & 2.1 & / & 0 & 3.5 & / &  0 & 3.3 & / & 0 & 3.8 & / & 0 & 0 & /	\\
 $n(1-\varepsilon)$	& 9.0 & 2.1 & / & 4.8 & 3.5 & / & 4.4 & 3.3 & / & 5.2 & 3.8 & / & 0 & 0 & /	\\
 \; \\
\textbf{MRCD} \\
 $H$ & 12.9 & 21.7 & / & 8.4 & 16.3 & / & 10.2 & 18.5 & / & 10 & 18.7 & / & 100 & 100 & /	\\
 $n(1-\varepsilon)$	& 27.5 & 21.7 & / & 18.9 & 16.3 & / & 19.7 & 18.5 & / & 20.1 & 18.7 & / & 100 & 100 & /	\\
\bottomrule
\end{tabular}
\end{adjustbox}
\end{table*}

\blue{We conclude that in this nonlinear 
setting, KMRCD has successfully extended 
the MRCD to non-elliptical distributions. 
We want to add two remarks about this.
First, as in all kernel-based methods the 
choice of the kernel is important, and 
choosing a different kernel can lead to
worse results. And second, just as in the
linear setting $h$ should be lower than
$n$ minus the number of outliers, so in
practice it is recommended to first run 
with a low $h$, look at the results in order
to find out how many outliers there are, and
possibly run again with a higher $h$.}

\blue{Section A.4 
of the supplementary material contains 
additional simulation results about
the computation time of the four initial
estimators in KMRCD and their subsequent 
C-steps, in different settings with
linear and nonlinear kernels.}

\section{Experiments} 
\label{sec:Rapplication}	
	
\subsection{Food industry example}
\label{sec:raisins}

We now turn our attention to a real dataset
from the food industry.
In that setting datasets frequently contain 
outliers, because samples originate from 
natural products which are often 
contaminated by insect damage, local 
discolorations and foreign material. 
It also happens that the image acquisition 
signals yield non-elliptical data, and in 
that case a kernel transform can help.

\begin{figure}[ht]
\centering
\includegraphics[width=1.0\textwidth]
  {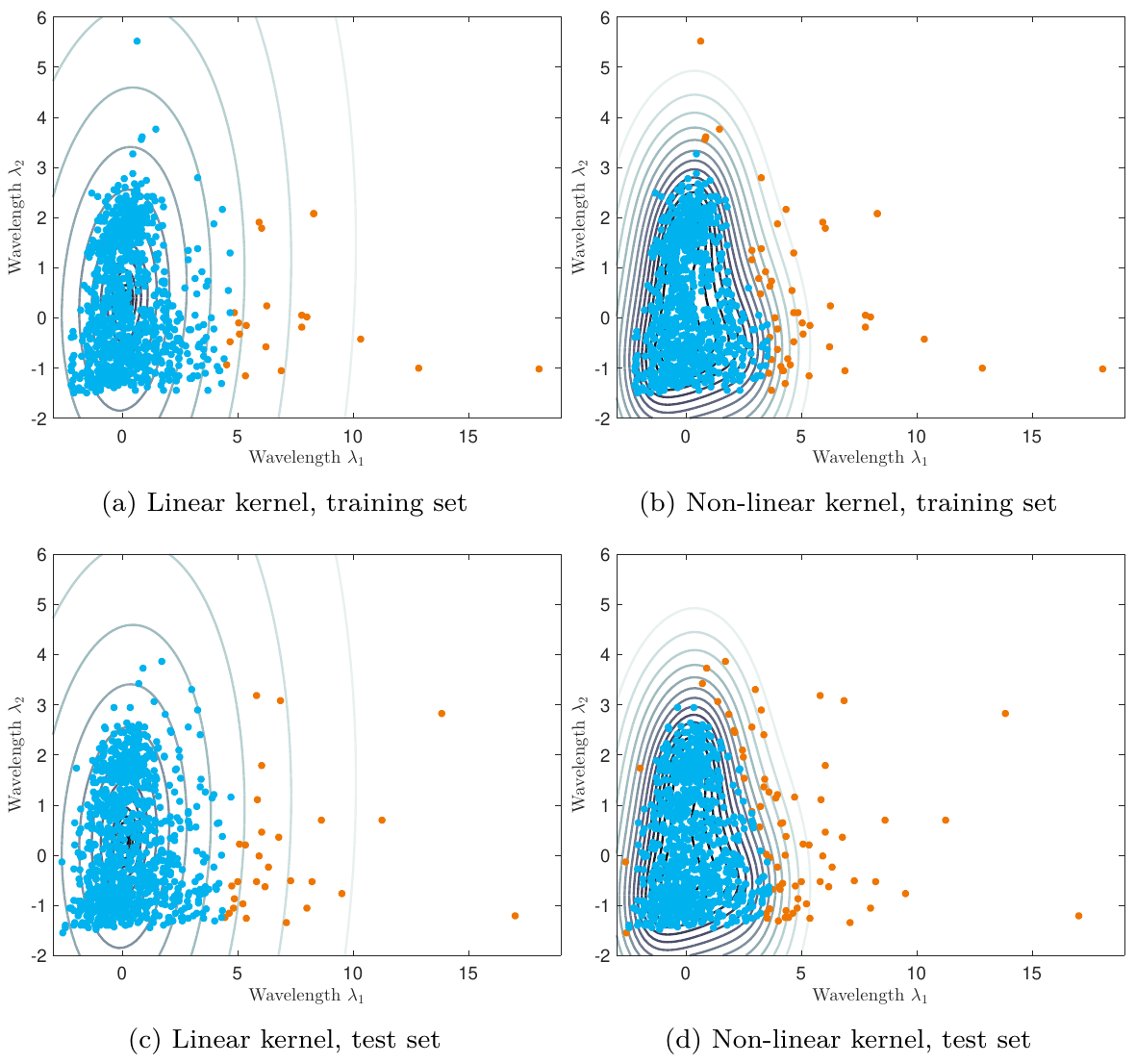}
\caption{Food industry example: KMRCD 
results with the linear kernel (left 
column) and the RBF kernel (right column). 
The top row contains the training data, 
and the resulting fits were applied to 
the test data in the bottom row.}
\label{fig:raisin}
\end{figure}

The dataset is bivariate and contains two
color signals measured on organic sultana 
raisin samples. 
The goal is to classify these into inliers
and outliers, so that during production 
outliers can be physically removed 
from the product in real time.
There are training data and test data,
but the class label `outlier' is not 
known beforehand.
The scatter plot of the training data in
Figure \ref{fig:raisin}~(a) reveals the 
non-elliptical (and to some extent 
triangular) structure of the inliers.
Three types of outliers are visible.
Those with high values of $\lambda_1$ 
and low $\lambda_2$ at the bottom right 
correspond to foreign, cap-stem related 
material like wood, whereas points 
with high values of $\lambda_2$ represent 
discolorations. 
There are also a few points with high
values of both $\lambda_1$ and 
$\lambda_2$ which correspond to either 
discolored raisins or objects with clear 
attachment points of cap-stems.
Outliers of any of these three types need
to be flagged and removed from the 
product. 
From manually analyzing data of this 
product it is known beforehand that the 
fraction of outliers is rather low, at 
most around $2\%$.

We first run KMRCD on the training data.
In its preprocessing step it standardizes
both variables.
For comparison purposes we use two 
kernels. In the left hand column of
Figure \ref{fig:raisin} we apply the
linear kernel, and in the right hand
column we use the RBF kernel with tuning 
constant $\sigma$ given by
\eqref{eq:heuristic}.
Since we know the fraction of outliers 
is low we can put $h = \floor{0.95n}$.
Each figure shows the flagged points in
orange and the remaining points in blue,
and the contour lines are level curves 
of the robust distance.

The fit with linear kernel in 
Figure \ref{fig:raisin}~(a) has contour
lines that do not follow the shape of 
the data very well, and as a consequence
it fails to flag some of the outliers, 
such as those with high $\lambda_2$ and 
some with relatively high values of both 
$\lambda_1$ and $\lambda_2$. 
The KMRCD fit with nonlinear kernel 
in Figure \ref{fig:raisin}~(b) has
contour lines that model the data more 
realistically. 
This fit does flag all three types of
outliers correctly.
Both trained models were then used to 
classify the previously unseen test set. 
The results are similar to those on the
training data.
The anomaly detection with linear kernel
in Figure \ref{fig:raisin}~(c) again misses 
the raisin discolorations, which would 
keep these impurities in the final 
consumer product.
Fortunately, the method with the
nonlinear kernel in panel (d) does
flag them.

\subsection{MNIST digits data}

Our last example is high dimensional.
The MNIST dataset contains images of 
handwritten digits from 0 to 9, at
the resolution of $28 \times 28$
grayscale pixels (so there are 784 
dimensions), and was downloaded from 
\url{http://yann.lecun.com/exdb/mnist}.
There is a training set and a test set.
Both were subsampled to 1000 images.
To the training data we added noise
distributed as $\mathcal{N}(0,(0.5)^2)$
to $20\%$ of the images, and in the
test set we added noise with the same
distribution to all images.
We then applied KMRCD with RBF kernel
with tuning constant $\sigma$ given by
\eqref{eq:heuristic} and 
$h=\floor{0.75n}$ to the 1000 training 
images.
Next, we computed the eigenvectors of 
the robustly estimated covariance 
matrix.

Our goal is to denoise the images in
the test set by projecting them onto
the main eigenvectors found in the
training data. 
As we are interested in a reconstruction 
of the data in the original space rather 
than in the feature space, we transform 
the scores back to the original input 
space by the iterative optimization 
method of \cite{mika1999kernel}.

\begin{figure}[ht]
\centering					
\includegraphics[width=\linewidth]
  {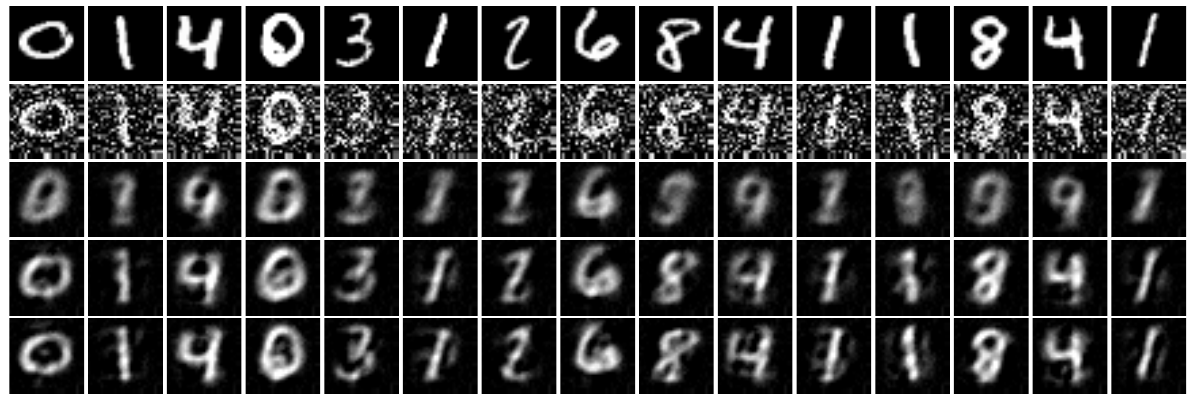}\\		
\vspace{0.4cm}
\includegraphics[width=\linewidth]
  {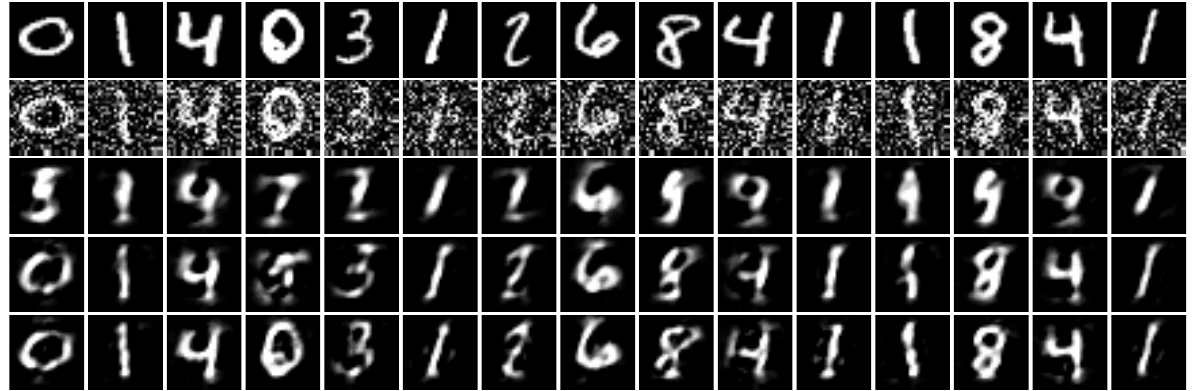} 			
\caption{MNIST denoising results based
on the classical covariance matrix
(top panel) and on the KMRCD estimate
(bottom panel). The first and second 
rows of each panel show the same 
original and noise-added test set 
images. The remaining rows contain the 
results of projecting on the first 5, 
15, and 30 eigenvectors.}
\label{fig:MNIST}
\end{figure}

The top panel of Figure \ref{fig:MNIST} 
illustrates what happens when applying 
this computation to the classical 
covariance matrix in feature space, which 
corresponds to classical kernel 
PCA \cite{scholkopf1998nonlinear}.
The bottom panel is based on KMRCD.
The first row of each panel displays 
original test set images, and the second 
row shows the test images after the noise 
was added.
The first and second rows do not depend on
the estimation method, but the remaining
rows do. There we see the results of 
projecting on the first 5, 15, and 30 
eigenvectors of each method.
In the top panel those images are rather 
diffuse, which indicates that the 
classical approach was affected by the 
training images with added noise and
considers the added noise as part of
its model. This implies that increasing
the number of eigenvectors used will not 
improve the overall image quality much. 
The lower panel contains sharper images, 
because the robust fit of the training 
data was less affected by the images with
added noise that acted as outliers.

We can also compute the mean absolute 
error
$\sum_{i=1}^n\sum_{j=1}^{p} |x_{i,p} -
 \hat{x}_{i,p}|/(np)$\linebreak
between the original test images (with
$p=784$ dimensions) and the projected
versions of the test images with added
noise.
Figure \ref{fig:mnistNUM} shows this 
deviation as a function of the number
of eigenvectors used in the projection.
The deviations of the robust method
are systematically lower than those of
the classical method, confirming the
visual impression from Figure
\ref{fig:MNIST}.

\begin{figure}[ht]
\centering
\includegraphics[width=0.5\textwidth]
  {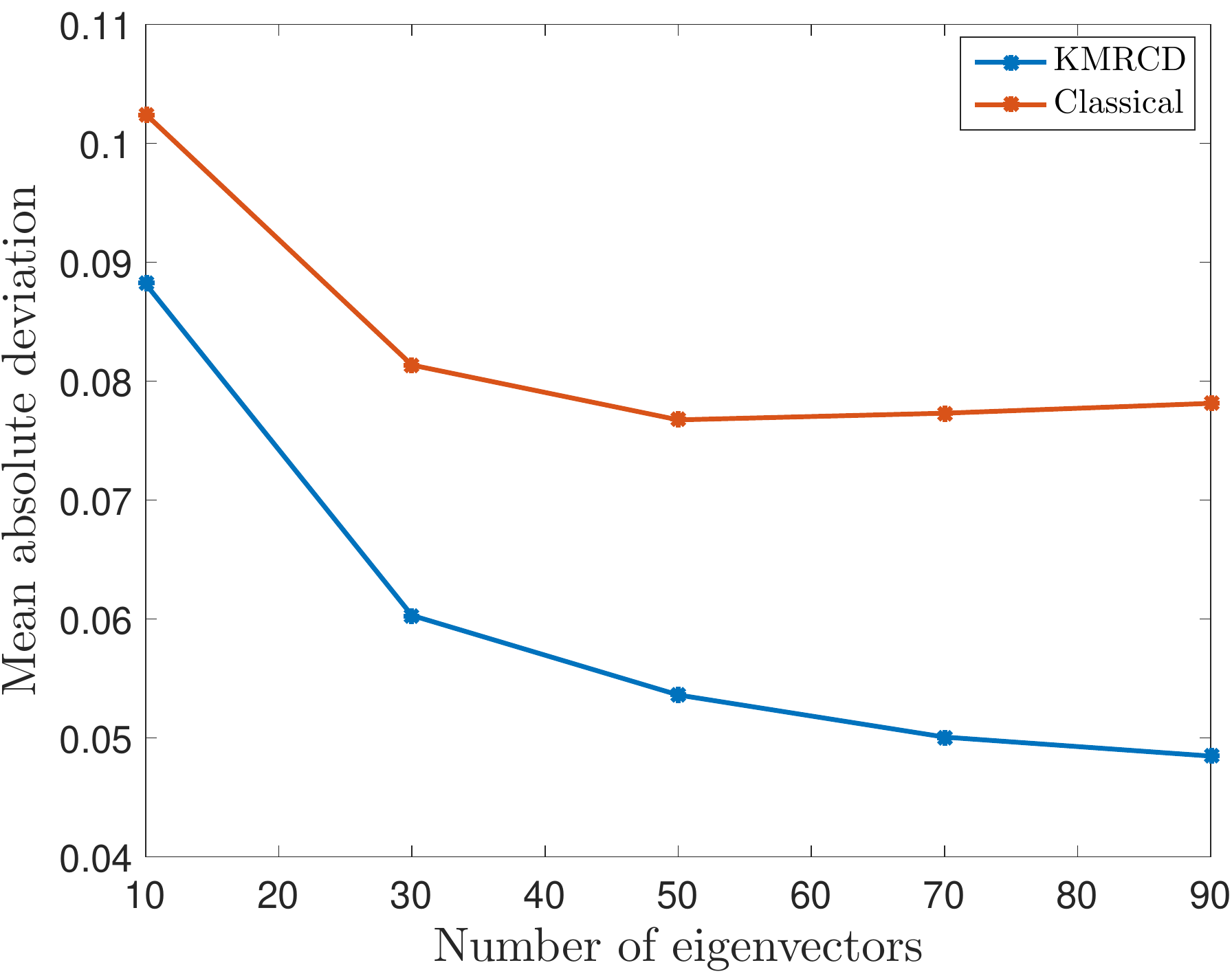}
\caption{The mean absolute error of the
denoised images to the original test images
in function of the number of eigenvectors
used in the projection. The top curve is 
for the classical covariance matrix in 
feature space, the lower curve for KMRCD.}
\label{fig:mnistNUM}
\end{figure}

\section{Conclusions} 
\label{sec:conclusion}	
	
The kernel MRCD method introduced in this
paper is a robust method that allows to
analyze non-elliptical data when used with
a nonlinear kernel. Another advantage
is that even when using the linear kernel
the computation becomes much faster when
there are more dimensions than cases, a
situation that is quite common nowadays.
Due to the built-in regularization the
result is always well-conditioned.

The algorithm starts from four kernelized
initial estimators, and to each it 
applies a new kernelized refinement step.
The remainder of the algorithm is based 
on a theorem showing that C-steps in
feature space are equivalent to a new
type of C-steps on centered kernel 
matrices, so the latter reduce the 
objective function.
The performance of KMRCD in terms of
robustness, accuracy and speed is 
studied in a simulation, and the method
is applied to several examples.
Potential future applications of the KMRCD 
method are as a building block for other 
multivariate techniques such as robust 
classification.

Research-level MATLAB code and an example 
script are freely available from the webpage 
\url{http://wis.kuleuven.be/statdatascience/robust/software}.\\

\noindent{\bf Acknowledgements.} We
thank Johan Speybrouck for providing 
the industrial dataset and Tim Wynants 
and Doug Reid for their support throughout 
the project. 
The research leading to these results has 
received funding from the European Research 
Council under the European Union's Horizon
2020 research and innovation program / ERC 
Advanced Grant E-DUALITY (787960). This 
paper reflects only the authors' views and 
the Union is not liable for any use that may 
be made of the contained information.
There was also support from the Research 
Council of KU Leuven (projects C14/18/068
and C16/15/068), the Flemish Government 
(VLAIO grant HBC.2016.0208 and FWO project 
GOA4917N on Deep Restricted Kernel Machines), 
and a PhD/Postdoc grant of the
Ford-KU Leuven Research Alliance Project 
KUL0076 (Stability analysis and performance 
improvement of deep reinforcement learning 
algorithms).




\clearpage
\pagenumbering{roman} 
%
\appendix
\numberwithin{equation}{section} 
\section{Supplementary Material} \label{sec:A}
\renewcommand{\theequation}
   {\thesection.\arabic{equation}}
	
\subsection{Algorithms} 
\label{sec:AppAlgo}

Algorithms 2 to 5 below
describe the initial estimators used.
Algorithm 6 
carries out the subsequent refinement 
procedure.

\begin{algorithm}[H]
\caption{Spatial median initial estimator}
\begin{enumerate}
\item \textbf{Input}: kernel matrix $K$ and
      the number $h$.
\item Initialize the vector 
      $\gamma = (1/n, 1/n, \dots, 1/n)^\top$
			with length $n$.
\item For $\mathrm{iteration} = 1$ to $10$ do
	\begin{enumerate}
	\item Update the coefficients $\gamma_i$ by
		\begin{equation*}
			 \gamma_{\mathrm{new},i} := ({K_{ii} - 
			2 \gamma_{\mathrm{old}}^T K_{.i} + 
			\gamma_{\mathrm{old}}^T 
			K \gamma_{\mathrm{old}}})^{-\frac{1}{2}}
		\end{equation*}
	      where $K_{ii}$ is the $i$-th diagonal 
				element of $K$ and $K_{.i}$ is the 
				$i$-th column of $K$. 
	\item Normalize $\gamma_{\mathrm{new}}
	      \leftarrow \gamma_{\mathrm{new}} / 
				(\sum_{i=1}^{n} 
				\gamma_{\mathrm{new},i})$\;.
	\end{enumerate}
\item Compute the distance of each observation 
      to the spatial median by 
			\eqref{eq:edistance}.
\item Construct $H$ as the set of the $h$ 
      observations with the lowest distances.	
\item Determine the weights $w_i$ using 
      \eqref{eq:hardReject}.
\item \textbf{Return:} the location weights 
      $w_i$ and the covariance weights 
			$u_i := w_i$\;. 
\end{enumerate}
\end{algorithm}

\begin{algorithm}[H]
\caption{Stahel-Donoho initial estimator}
\begin{enumerate}
\item \textbf{Input}: kernel matrix $K$ and
      the number $h$.
\item Initialize the vector of outlyingness
      values of the $n$ observations as 
			$\eta = (0, \dots, 0)^\top$.
\item For $q=1:500$ (number of directions) do
	\begin{enumerate}
	\item Select indices $i \neq j$ at random 
	      from $\{1,\ldots,n\}$
	\item Initialize the vector 
	      $\lambda = (0, \dots, 0)$ of length
				$n$ and set $\lambda(i) = +1$ and
				$\lambda(j) = -1$. 
				This vector represents the direction
				in $\mathcal{F}$ formed by two 
				observations. 
	\item Compute the projections of all $n$
	      points on this direction as 
	      $a=K\lambda/
				\sqrt{\lambda^\top K\lambda}$\;.
	\item Compute the outlyingness of all
	      $n$ projected points as 
				$r^{sd}=|a-\med(a)|/\mad(a)$, 
				where $\mad(.)$ is the median 
				absolution deviation.
	\item Update the maximum outlyingness of 
	      each observation: 
				$\eta_i^\mathrm{new} = 
				\max(\eta_i^\mathrm{old}, r^{d}_i)$.
	\end{enumerate}
\item Define $H$ as the set of $h$ observations 
      with the lowest $\eta_i$\;.	
\item Determine the weights $w_i$ using 
      \eqref{eq:hardReject}.
\item \textbf{Return:} the location weights 
      $w_i$ and the covariance weights 
			$u_i := w_i$\;.
\end{enumerate}
\end{algorithm}

\begin{algorithm}[H]
\caption{Spatial rank initial estimator}
\begin{enumerate}
\item \textbf{Input}: kernel matrix $K$ and
	      the number $h$.
\item Initialize the vector with the $n$ 
      spatial ranks as 
			$R = (0, \dots, 0)^\top$.
\item Compute the spatial rank $R_i$ of 
      each observation using \eqref{eq::srank}.
\item Define $H$ as the set of $h$ 
      observations with the lowest values 
			of $R_i$.	
\item Determine the weights $w_i$ using 
      \eqref{eq:hardReject}.
\item \textbf{Return:} the location weights 
      $w_i$ and the covariance weights 
			$u_i := w_i$\;. 			
\end{enumerate}
\end{algorithm}

\begin{algorithm}[ht]
\caption{Spatial sign covariance matrix 
         initial estimator}
\begin{enumerate}
\item \textbf{Input}: kernel matrix $K$.
\item Compute the spatial median $m_\mF = 
      \sum_{i=1}^N \gamma_i \phi(x_i)$ 
			using Algorithm 2. 
\item The location weights are
      taken as $w_i= \gamma_i$\;.
\item Compute the distances
      $||\phi(x_i) - m_\mF||$ to the 
			spatial median from 
			\eqref{eq:edistance}.
			The \mbox{covariance} weights are 
			then computed as
			$u_i = 1/||\phi(x_i) - m_\mF||$.
\item \textbf{Return:} the location 
      weights $w_i$ and the covariance 
			weights $u_i$\;. 
\end{enumerate}
\end{algorithm}

\vspace{0.2cm}

\begin{algorithm}[ht]
\caption{Refinement of initial estimators}
\begin{enumerate}
\item \textbf{Input}: kernel matrix $K$ and 
      number $h$, weights $w_i$ to 
			determine the initial estimate 
			of location $c_\mF$ by 
			\eqref{eq:wsumloc} and weights $u_i$
			to compute the initial
			estimate of scatter $\hSigma_\mF$ by 
			\eqref{eq:wsumscatter}.
\item Project the data on the eigenvectors 
      of $\hSigma_\mF$ using 
			\eqref{eq:refin_B}.
\item Compute the spatial median of the 
      adjusted kernel matrix 
			\eqref{eq:refin_Kadj} using 
			Algorithm 2, 
			which gives the weights $\gamma^*$. 
\item Determine the distance 
      $d^{*}_\mF(x)$ of each 
			observation by \eqref{eq:refin_dist}.
\item Define $H$ as the set of $h$ 
      observations with the lowest distance.	
\item \textbf{Output}: the refined subset $H$.
\end{enumerate}
\end{algorithm} 

\FloatBarrier
\clearpage

\subsection{The special case of the linear kernel} 
\label{sec:linearkernel}

Given an $n \times p$ matrix $X$ of observations,
its regularized covariance matrix is 
given by
\begin{equation*}
	\hSigma_{\mathrm{reg}} =
	(1-\rho)\hSigma_\mF + \rho I_\mF =
	\frac{1-\rho}{n-1} \tilde{\Phi} 
	\tilde{\Phi}^\top + \rho I_\mF\;.
\end{equation*}
The regularized kernel matrix is defined as
\begin{equation*}
	\tilde{K}_{\mathrm{reg}} = 
	(1-\rho)\tilde{K} + (n-1) \rho I_n =
	(1-\rho)\tilde{\Phi}^{\top} \tilde{\Phi} + 
	(n-1) \rho I_n.
\end{equation*}
For the linear kernel, the eigenvectors of 
the covariance matrix are given by
\begin{equation*}
	v_\mathcal{X}^k = \sum_{i=1}^{n} 
	(v_\mF^k)_i (x_i - c_\mF), 
\end{equation*}
where $c_\mF = \frac{1}{n}\sum_{i=1}^n x_i$ 
and $(v_\mathcal{F}^k)_i$ is the $i$-th 
element of the $k$-th eigenvector of the 
centered kernel matrix. 
We can now write $\hSigma_\mF$ using the 
eigenvectors of the kernel matrix:
\begin{equation*}
	\hSigma_\mF = V^T \Lambda V =
	\frac{1}{n-1}\tilde{X}^\top
	\tilde{V}_\mF \Lambda \tilde{V}_\mF^\top 
	\tilde{X},
\end{equation*}
where the $n \times n$ matrix 
$\tilde{V}_\mF = \left[
 v_1/\sqrt{\lambda_1},\ldots,
 v_n/\sqrt{\lambda_n}\right]$ contains the 
normalized eigenvectors of the centered 
kernel matrix and $\tilde{X}$ is the 
centered data matrix. The regularized 
covariance matrix is thus equal to
\begin{equation*}
	\hSigma_{\mathrm{reg}} = 
	\frac{1-\rho}{n-1}\tilde{X}^\top
	\tilde{V}_\mF \Lambda \tilde{V}_\mF^\top 
	\tilde{X} + \rho I_n\;.
\end{equation*}	

\clearpage
\subsection{\blue{Effect of the imposed
            condition number}}

\blue{The main text considers the condition 
number $\kappa$  of a covariance matrix in two 
different places. 
The first occurrence is in the KMRCD algorithm
itself. Indeed, the choice of the regularization
parameter $\rho$ in~\eqref{eq:egKernelMatrix}
must be such that the condition number
$\kappa(\rho)$ given 
by~\eqref{eq:conditionNumber} is at most 50. 
The second occurrence is in the simulation 
study with linear kernel, when in each
replication a random correlation matrix 
$\Sigma$ of type ALYZ is generated according
to Section 4 of \cite{agostinelli2015robust},
who impose that such correlation matrices
have condition number 
$\kappa(\mathrm{ALYZ}) = 100$\,.}

\blue{It is natural to wonder about the effect
of both choices, and to what extent they 
interact.
To study this we repeat part of the simulation
of Section~\ref{subsec:Lin_simulation}, for
the hardest settings with $(n,p)$ set to
$(400,200)$ and $(200,400)$ and with
$\varepsilon$ equal to 0 and 0.3 and $h = 0.5$.
The contamination types are as before. 
We now allow the condition number 
$\kappa(\mathrm{ALYZ})$ to be $100$, $300$ 
and $1000$. 
The imposed condition number $\kappa(\rho)$ in
the algorithm can now take the values $50$, 
$100$ and $1000$. 
The averaged KL divergence over 100 
replications is shown in 
Table~\ref{tbl:conditionNumber_KL}.} 

\begin{table}[ht]
\caption{\blue{The KL divergence in function 
of the condition number of the ALYZ covariance 
matrix $\kappa(\mathrm{ALYZ})$ and 
$\kappa(\rho)$ of the KMRCD algorithm 
with $h = 0.5$.}} 
\label{tbl:conditionNumber_KL}
\centering               
\begin{adjustbox}{width=1\textwidth}                        
\centering
\begin{tabular}{lcccccccccc}
\toprule                 
 & & \multicolumn{3}{c}{Point contamination} & \multicolumn{3}{c}{Shift contamination}  & \multicolumn{3}{c}{Cluster contamination}\\ 
&	& \multicolumn{3}{c}{Value of $\kappa(\mathrm{ALYZ})$} & \multicolumn{3}{c}{Value of $\kappa(\mathrm{ALYZ})$}  & \multicolumn{3}{c}{Value of $\kappa(\mathrm{ALYZ})$}\\ 	
				\cmidrule(lr){3-5} \cmidrule(lr){6-8} \cmidrule(lr){9-11}
& & 100 & 300 & 1000 & 100 & 300 & 1000 & 100 & 300 & 1000\\
& \textbf{$\varepsilon=0$\,:} \\
\midrule
\multirow{3}{*}{\rotatebox{90}{\tiny 400$\times$200}} & $\kappa(\rho) = 50$	&127 & 146 & 197 & 128 & 147 & 198 & 127 & 145 & 197 \\
& $\kappa(\rho) = 100$	&158 & 166 & 196 & 159 & 166 & 197 & 159 & 167 & 194\\
& $\kappa(\rho) = 1000$	&269 & 264 & 265 & 268 & 265 & 264 & 268 & 264 & 264 \\ 
\cmidrule(lr){1-11}
\multirow{3}{*}{\rotatebox{90}{\tiny 200$\times$400}} & $\kappa(\rho) = 50$	& 674 &	761	& 969&	677	&759&	969	&675&	759&	943 \\
& $\kappa(\rho) = 100$	&1104 & 1168 & 1340 & 1108 & 1169 & 1326 & 1107 & 1169 & 1323\\
& $\kappa(\rho) = 1000$ & 6480 & 6380 & 6401 & 6492 & 6366 & 6367 & 6440 & 6375 & 6403 \\ 
\; \\ 
 & \textbf{$\varepsilon=0.3$\,:} \\
  \cmidrule(lr){1-11}
\multirow{3}{*}{\rotatebox{90}{\tiny 400$\times$200}} & $\kappa(\rho) = 50$	& 125 & 142   & 11536 & 125 & 143 & 205 & 125 & 144 & 251 \\
& $\kappa(\rho) = 100$	& 156 & 163   & 20581 & 155 & 164 & 193 & 155 & 164 & 422 \\
& $\kappa(\rho) = 1000$	& 265 & 10189 & 84460 & 256 & 185 & 209 & 257 & 189 & 209 \\ 
\cmidrule(lr){1-11}
\multirow{3}{*}{\rotatebox{90}{\tiny 200$\times$400}} & $\kappa(\rho) = 50$	& 694  & 776   & 13729  & 692  & 786  & 1007 & 693  & 771  & 1041 \\
& $\kappa(\rho) = 100$	& 1138 & 1199  & 25464  & 1157 & 1216 & 1416 & 1152 & 1213 & 1344 \\
& $\kappa(\rho) = 1000$	& 6766 & 72332 & 195728 & 6906 & 1570 & 5405 & 6862 & 1568 & 5275 \\
\bottomrule
\end{tabular}
\end{adjustbox}
\end{table}

\blue{Comparing the choices of $\kappa(\rho)$, 
we observe that the strongest regularization
($\kappa(\rho) = 50$) consistently performed
best (had the lowest KL), across all 
scenarios considered.
That is, even when the true 
$\kappa(\mathrm{ALYZ})$ is substantially
above $50$, setting $\kappa(\rho)$ to $50$ 
did better or equally well than setting it 
to $100$ or $1000$.
This provides some support for the choice
$\kappa(\rho)=50$ in the KMRCD algorithm.}

\clearpage

\subsection{\blue{Comparison of initial 
            estimators}}
\label{subsec:InitialEst}

\blue{In this section we study the computation
time and the performance of the four initial 
estimators used by KMRCD, in several scenarios:
\begin{enumerate}
\item as in the simulation with linear kernel
  (Section~\ref{subsec:Lin_simulation}), with
	point contamination, $\varepsilon = 0.3$, 
	$h = 0.5$, and $(n,p) = (400,200)$;
\item as in the first setting, but now with
  $(n,p) = (200,400)$ so there are more
	variables than cases;
\item the $t$-copula setting of the 
  simulation with nonlinear kernels, as in
	Figure~\ref{fig:nonLinSimulation_T} of
	section~\ref{subsec:NonLin_simulation}, 
	with $\varepsilon = 0.2$ and  $h=0.75$.
\item the circle-based setting of the 
  simulation with nonlinear kernels, as in 
	Figure~\ref{fig:nonLinSimulation_Circle},
	also with $\varepsilon = 0.2$ and $h=0.75$.
\end{enumerate}
Each setting is replicated $100$ times, and
in each replication we monitor the computation 
time of the individual initial estimators as 
well as the time needed by the subsequent 
C-steps procedure. The averaged computation 
times (in seconds) are given in 
Table~\ref{tbl:InitialEst_time}.}

\begin{table}[ht]
\caption{\blue{Averaged computation times of 
the four initial estimators (init) and the 
subsequent C-steps, in the four data 
settings described in 
Section~\ref{subsec:InitialEst}.}} 
\label{tbl:InitialEst_time}
\centering               
\centering
\begin{tabular}{cccccccccc}
\toprule                 
 & \multicolumn{2}{c}{Spatial median} 
 & \multicolumn{2}{c}{SDO} 
 & \multicolumn{2}{c}{Spatial rank}  
 & \multicolumn{2}{c}{SSCM}\\ 
 \cmidrule(lr){2-3} \cmidrule(lr){4-5} 
 \cmidrule(lr){6-7} \cmidrule(lr){8-9}
 & init & C-steps & init & C-steps
 & init & C-steps & init & C-steps\\
\midrule
Setting 1 & 0.020 & 0.002 
 & 0.046 & 0.002 
 & 0.406 & 0.002 
 & 0.034 & 0.002 \\
Setting 2 & 0.007 & 0.001 
 & 0.029 & 0.001 
 & 0.070 & 0.001 
 & 0.010 & 0.001 \\
Setting 3 & 0.024 & 0.024 
 & 0.055 & 0.024 
 & 0.667 & 0.025 
 & 0.037 & 0.024  \\
Setting 4 & 0.023 & 0.018
 & 0.052 & 0.028 
 & 0.670 & 0.027  
 & 0.035 & 0.034 \\
\bottomrule
\end{tabular}
\end{table}

\blue{In Table~\ref{tbl:InitialEst_time}
we see that the C-steps never take longer
than their initial estimator, and that
the computation times of the C-steps are
similar across the four initial estimates.
Among the initial estimators, the spatial
rank took the longest, whereas the other
three took about the same time.
Also note that the bivariate settings 3
and 4 are not faster than setting 1 because 
their sample size is $n=500$, so the 
computations are done on kernel matrices
of size $500 \times 500$.}

\blue{We also counted the number of times 
each initial estimator provided the best 
solution (i.e. the lowest covariance 
determinant) after its C-steps converged.
The results are given in 
Table~\ref{tbl:InitialEst_count}.
Note that two or more estimators can give 
the same best solution, so the row sums
can exceed 100.
In settings 1 and 2 of high-dimensional 
data analyzed with the linear kernel, the 
SSCM initial estimator outperformed the 
others. In settings 3 and 4 with 
nonlinear kernels, all four initial
estimators performed about equally well.}

\begin{table}[ht]
\caption{\blue{The number of times each 
initial estimator provided the best solution 
(i.e. the lowest covariance determinant) 
after C-steps, for the 4 settings in 
Table~\ref{tbl:InitialEst_time}.}}
\label{tbl:InitialEst_count}
\centering               
\centering
\begin{tabular}{ccccc}
\toprule                 
 & Spatial median & SDO 
 & Spatial rank   & SSCM\\ 
\midrule
Setting 1 & 4  & 0  & 1  & 95 \\
Setting 2 & 0  & 5  & 0  & 95 \\
Setting 3 & 57 & 65 & 52 & 66 \\
Setting 4 & 52 & 60 & 58 & 53 \\
\bottomrule
\end{tabular}
\end{table}

\FloatBarrier
\subsection{Additional simulation results
            with linear kernel}	
	
Table \ref{tbl:MSE} shows the mean squared
error (MSE) of the estimates $\hSigma$ in the 
same setup as Table \ref{tbl:KL} 
of the main text. For readability the MSE
values are multiplied by 1000.
Table \ref{tbl:rho} lists the 
average values of the regularization 
parameter $\rho$ in each setting.
	
\begin{table}[ht] 
\caption{MSE for $\Sigma$ of type ALYZ 
         (multiplied by $1000$).} 
\label{tbl:MSE}
\centering               
\begin{adjustbox}{width=0.79\textwidth}                        
\centering
\begin{tabular}{cccccccccccc}
\toprule                 
  & \multicolumn{3}{c}{Point contamination} & \multicolumn{3}{c}{Shift contamination}  & \multicolumn{3}{c}{Cluster contamination}\\ 
	& \multicolumn{3}{c}{Value of $h/n$} & \multicolumn{3}{c}{Value of $h/n$}  & \multicolumn{3}{c}{Value of $h/n$}\\ 	
				\cmidrule(lr){2-4} \cmidrule(lr){5-7} \cmidrule(lr){8-10}
 & 0.50 & 0.75 & 0.90 & 0.50 & 0.75 & 0.90 & 0.50 & 0.75 & 0.90\\
\textbf{$\varepsilon=0$\,:} \\
\midrule
\textbf{KMRCD} \\
			400$\times$200	&4.31	& 2.96	& 2.47	& 4.33	& 2.94	& 2.46	& 4.35	& 2.94	& 2.43		\\	
			300$\times$200	& 5.51	& 3.81	& 3.19	& 5.44	& 3.80	& 3.18	& 5.52	& 3.79	& 3.21	\\	
			200$\times$200	& 7.88	& 5.47	& 4.66	& 7.78	& 5.54	& 4.71	& 7.77	& 5.42	& 4.71	\\	
			200$\times$300	& 7.36	& 5.21	& 4.45	& 7.29	& 5.21	& 4.45	& 7.37	& 5.22	& 4.45	\\	
			200$\times$400	& 6.95	& 5.04	& 4.32	& 6.91	& 4.99	& 4.31	& 6.94	& 5.11	& 4.29		\\	
\; \\
\textbf{MRCD} \\
			400$\times$200	& 4.32	& 2.99	& 2.51	& 4.32	& 2.96	& 2.51	& 4.34	& 2.99	& 2.48		\\	
			300$\times$200	& 5.59	& 3.83	& 3.22	& 5.51	& 3.83	& 3.21	& 5.54	& 3.83	& 3.26	\\	
			200$\times$200	& 7.89	& 5.60	& 4.67	& 7.85	& 5.59	& 4.69	& 7.75	& 5.61	& 4.70		\\	
			200$\times$300	& 7.42	& 5.28	& 4.51	& 7.42	& 5.28	& 4.49	& 7.47	& 5.33	& 4.52	\\	
			200$\times$400	& 7.04	& 5.04	& 4.33	& 7.08	& 5.04	& 4.36	& 7.01	& 5.09	& 4.31	\\	
\; \\
\textbf{$\varepsilon=0.1$\,:} \\
\midrule
\textbf{KMRCD} \\
			400$\times$200	& 4.35	& 2.99	& 2.53	& 4.38	& 3.00	& 2.51	& 4.37	& 2.99	& 2.50	\\
			300$\times$200	& 5.61	& 3.86	& 3.29	& 5.59	& 3.90	& 3.28	& 5.60	& 3.89	& 3.27	\\
			200$\times$200	& 7.97	& 5.60	& 4.86	& 7.93	& 5.58	& 4.84	& 7.99	& 5.64	& 4.83	\\
			200$\times$300	& 7.38	& 5.31	& 4.59	& 7.44	& 5.32	& 4.56	& 7.39	& 5.30	& 4.60	\\
			200$\times$400	& 7.09	& 5.13	& 4.39	& 7.03	& 5.11	& 4.41	& 7.04	& 5.10	& 4.43	\\
			\; \\
\textbf{MRCD} \\ 
			400$\times$200	& 4.43	& 3.05	& 2.58	& 4.52	& 3.12	& 2.60	& 4.49	& 3.11	& 2.59	\\
			300$\times$200	& 5.75	& 3.95	& 3.34	& 5.85	& 4.05	& 3.35	& 5.88	& 4.04	& 3.36	\\
			200$\times$200	& 8.33	& 5.74	& 4.94	& 8.18	& 5.81	& 4.97	& 8.20	& 5.86	& 4.96	\\
			200$\times$300	& 7.68	& 5.45	& 4.62	& 7.78	& 5.57	& 4.69	& 7.78	& 5.52	& 4.72	\\
			200$\times$400	& 7.22	& 5.15	& 4.38	& 7.44	& 5.28	& 4.52	& 7.41	& 5.31	& 4.55	\\
\; \\
\textbf{$\varepsilon=0.3$\,:} \\
\midrule
\textbf{KMRCD} \\
			400$\times$200	& 4.55	&247.32	& 58.13	& 4.55	& 18.74	& 40.78	& 4.54	& 18.75	& 39.88	\\	
			300$\times$200	& 5.80	&242.03	& 58.59	& 5.88	& 18.34	& 38.22	& 5.80	& 18.47	& 39.34	\\	
			200$\times$200	& 8.28	&246.56	& 59.43	& 8.51	& 19.15	& 37.54	& 8.37	& 19.09	& 37.85	\\	
			200$\times$300	& 7.78	&109.01	& 26.17	& 7.88	& 9.43	& 16.42	& 7.87	& 9.43	& 16.47	\\	
			200$\times$400	& 7.41	& 61.23	& 15.02	& 7.40	& 6.08	& 9.44	& 7.45	& 6.13	& 9.49	\\
\; \\
\textbf{MRCD} \\
			400$\times$200	& 4.16	&329.39	& 89.32	& 4.85	& 37.14	&115.04	& 4.87	& 36.75	&114.14	\\	
			300$\times$200	& 5.25	&312.86	& 90.28	& 6.39	& 37.43	&109.50	& 6.31	& 37.72	&114.09	\\	
			200$\times$200	& 7.44	&326.92	& 92.79	& 9.07	& 38.99	&108.50		& 8.83& 39.07	&110.41		\\
			200$\times$300	& 6.49	&128.66	& 33.82	& 8.52	& 17.44	& 42.28	& 8.56	& 17.53	& 42.15	\\	
			200$\times$400	& 5.62	& 64.31	& 16.70	& 7.97	& 10.51	& 21.52	& 8.03	& 10.75	&21.83	\\	
\bottomrule
\end{tabular}
\end{adjustbox}
\end{table}

\begin{table}[ht] 
\caption{Regularization coefficients for $\Sigma$ of type ALYZ.} 
\label{tbl:rho}
\centering               
\begin{adjustbox}{width=0.8\textwidth}                        
\centering
\begin{tabular}{cccccccccccc}
\toprule  
  & \multicolumn{3}{c}{Point contamination} & \multicolumn{3}{c}{Shift contamination}  & \multicolumn{3}{c}{Cluster contamination}\\ 
	& \multicolumn{3}{c}{Value of $h/n$} & \multicolumn{3}{c}{Value of $h/n$}  & \multicolumn{3}{c}{Value of $h/n$}\\ 	
				\cmidrule(lr){2-4} \cmidrule(lr){5-7} \cmidrule(lr){8-10}
 & 0.50 & 0.75 & 0.90 & 0.50 & 0.75 & 0.90 & 0.50 & 0.75 & 0.90\\
\textbf{$\varepsilon=0$}\,: \\
\midrule										
\textbf{KMRCD} \\
			400$\times$200	&0.091	&0.075	&0.069	&0.093	&0.073	&0.071	&0.088	&0.077	&0.072	\\
			300$\times$200	&0.094	&0.086	&0.079	&0.097	&0.089	&0.079	&0.097	&0.088	&0.087		\\
			200$\times$200	&0.117	&0.096	&0.091	&0.119	&0.100	&0.091	&0.117	&0.098	&0.092	\\
			200$\times$300	&0.155	&0.116	&0.107	&0.155	&0.116	&0.109	&0.162	&0.116	&0.108	\\
			200$\times$400	&0.179	&0.143	&0.125	&0.178	&0.139	&0.124	&0.178	&0.138	&0.126	\\
\; \\
\textbf{MRCD} \\
			400$\times$200	&0.091	&0.073	&0.063	&0.093	&0.070	&0.064	&0.092	&0.073	&0.064	\\
			300$\times$200	&0.100	&0.084	&0.078	&0.098	&0.087	&0.075	&0.099	&0.085	&0.077	\\
			200$\times$200	&0.115	&0.099	&0.094	&0.115	&0.102	&0.094	&0.115	&0.102	&0.092	\\
			200$\times$300	&0.143	&0.116	&0.107	&0.140	&0.119	&0.108		&0.145	&0.117	&0.107	\\
			200$\times$400	&0.164	&0.139	&0.125	&0.169	&0.138	&0.125	&0.165	&0.134	&0.124	\\	
\; \\
\textbf{$\varepsilon=0.1$\,:} \\
\midrule
\textbf{KMRCD} \\
			400$\times$200	&0.089	&0.079	&0.074	&0.089	&0.076	&0.074	&0.095	&0.080	&0.069	\\
			300$\times$200	&0.098	&0.085	&0.079	&0.097	&0.089	&0.083	&0.098	&0.088	&0.079	\\
			200$\times$200	&0.424	&0.098	&0.097	&0.120	&0.099	&0.090	&0.120	&0.098	&0.092	\\
			200$\times$300	&0.156	&0.120	&0.110	&0.153	&0.118	&0.112	&0.150	&0.118	&0.112	\\
			200$\times$400	&0.181	&0.139	&0.127	&0.176	&0.140	&0.126	&0.176	&0.140	&0.125	\\
\; \\
\textbf{MRCD} \\
			400$\times$200	&0.087	&0.073	&0.064	&0.082	&0.061	&0.059	&0.081	&0.065	&0.054	\\
			300$\times$200	&0.096	&0.081	&0.073	&0.092	&0.076	&0.069	&0.089	&0.073	&0.065	\\
			200$\times$200	&0.115	&0.096	&0.091	&0.102	&0.088	&0.078	&0.103	&0.088	&0.081	\\
			200$\times$300	&0.146	&0.118	&0.109	&0.133	&0.105	&0.098	&0.129	&0.107	&0.100	\\
			200$\times$400	&0.172	&0.138	&0.130	&0.153	&0.125	&0.115	&0.154	&0.124	&0.114	\\
\; \\
\textbf{$\varepsilon=0.3$\,:} \\
\midrule
\textbf{KMRCD} \\
			400$\times$200	&0.094	&0.524	&0.740	&0.097	&0.507	&0.733	&0.091	&0.517	&0.744	\\
			300$\times$200	&0.100	&0.505	&0.738	&0.101	&0.527	&0.746	&0.101	&0.521	&0.742	\\
			200$\times$200	&0.124	&0.507	&0.746	&0.119	&0.511	&0.749	&0.118	&0.521	&0.745	\\
			200$\times$300	&0.151	&0.508	&0.749	&0.153	&0.516	&0.748	&0.155	&0.504	&0.748	\\
			200$\times$400	&0.178	&0.512	&0.740	&0.175	&0.510	&0.737		&0.174	&0.518	&0.748	\\
\; \\
\textbf{MRCD} \\ 
			400$\times$200	&0.534	&0.644	&0.705	&0.071	&0.308	&0.541	&0.068	&0.314	&0.555	\\
			300$\times$200	&0.459	&0.745	&0.724	&0.078	&0.325	&0.557	&0.080	&0.323	&0.558	\\
			200$\times$200	&0.750	&0.688	&0.712	&0.092	&0.319	&0.565	&0.089	&0.326		&0.564	\\
			200$\times$300	&0.774	&0.770	&0.760	&0.116	&0.349	&0.591	&0.118	&0.338	&0.592	\\
			200$\times$400	&0.796	&0.790	&0.775	&0.139	&0.364	&0.610	&0.142	&0.372	&0.616	\\
\bottomrule
\end{tabular}
\end{adjustbox}
\end{table}

\clearpage
\newpage

\subsection{\blue{Additional figures 
  for non-elliptical data}}

\blue{Here we illustrate the Frank, Clayton, 
and Gumbel 
copulas~\cite{nelsen2007introduction} 
with Kendall rank correlation 
$\tau = 0.6$\,. 
Figure \ref{fig:SMnonLinSimulation}
below is analogous to 
Figure~\ref{fig:nonLinSimulation_T}
for the $t$ copula in the main text.}

\begin{figure}[ht]
\centering
\includegraphics[width=0.88\textwidth]
  {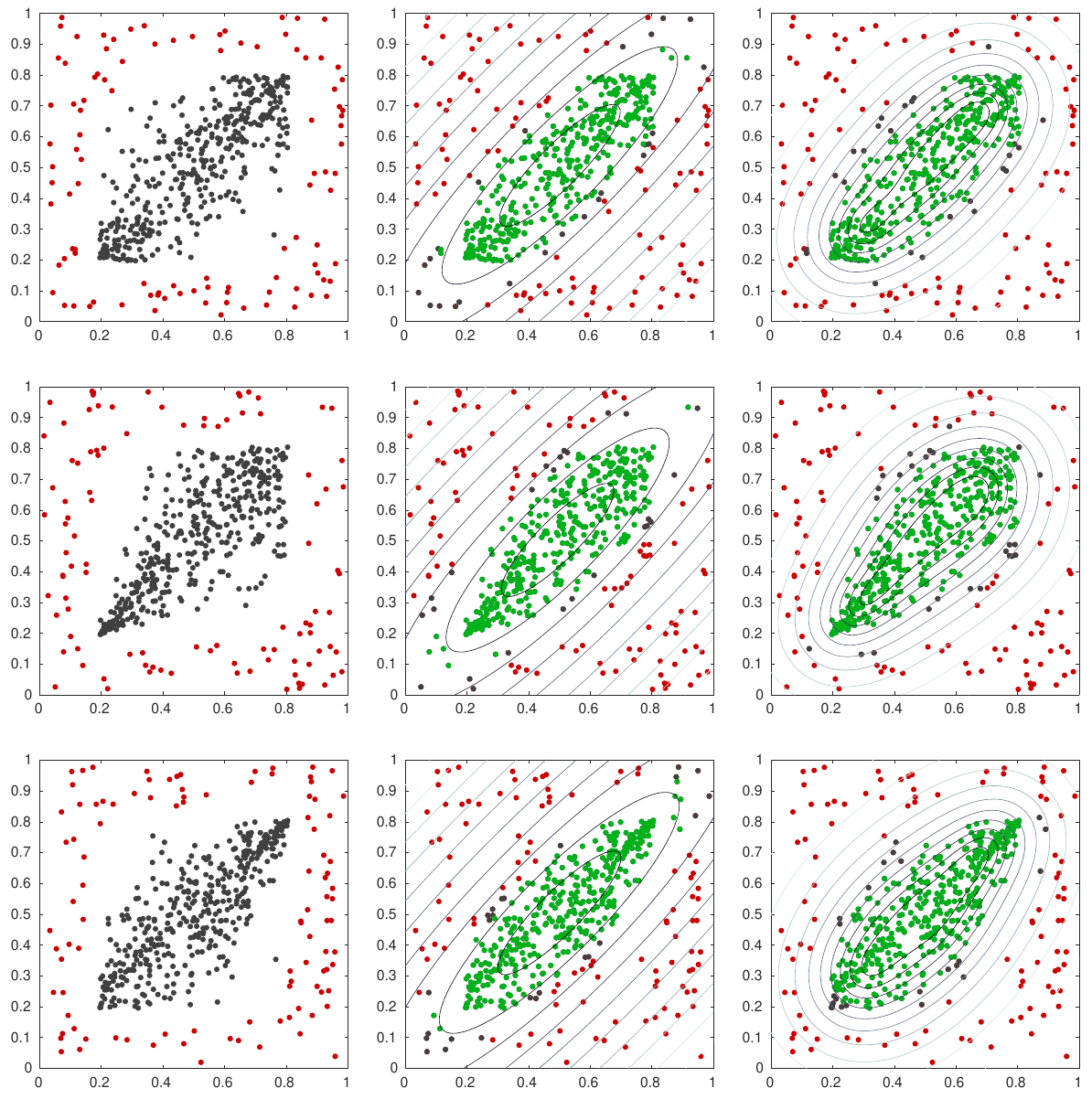}
\vspace{-2mm}	
\caption{\blue{Illustration of the non-elliptical 
simulation setting, on data sets generated from 
the Frank (top), Clayton (middle), and Gumbel 
(bottom) copulas. 
Each dataset contains $20\%$ of outlying 
observations. In the left panels, the regular 
observations are shown in black and the outliers 
in red. The results of the MRCD estimator are in 
the middle panels, and those of the KMRCD 
estimator in the rightmost panels, each for
$h = 0.75n$. In those panels the points in the 
h-subset are shown in green, and the other
points with the $n(1-\varepsilon)$ lowest 
(kernel) Mahalanobis distance are depicted in 
grey. The remaining points are shown in red. 
The curves are contours of the 
robust (kernel) Mahalanobis distance.}}
\label{fig:SMnonLinSimulation}
\end{figure}

\end{document}